%% file: main.tex
\documentclass[12pt]{article}
\usepackage{float}
\usepackage{amsmath}
\usepackage{amsthm}
\usepackage{amssymb}
\usepackage{fullpage}
\usepackage{graphicx}
\usepackage[numbers]{natbib}
\usepackage{url}
\usepackage{epstopdf}
\makeatletter

\floatstyle{ruled}
\newfloat{algorithm}{tbp}{loa}
\providecommand{\algorithmname}{Algorithm}
\floatname{algorithm}{\protect\algorithmname}

\theoremstyle{plain}
\newtheorem{prop}{\protect\propositionname}
\theoremstyle{plain}
\newtheorem{thm}{\protect\theoremname}
\theoremstyle{plain}
\newtheorem{assumption}{\protect\assumptionname}
\theoremstyle{plain}
\newtheorem{lem}{\protect\lemmaname}
\theoremstyle{plain}

\@ifundefined{date}{}{\date{}}
\usepackage{booktabs}
\usepackage{amsfonts}
\usepackage{nicefrac}
\usepackage{microtype}

\usepackage{lipsum}
\usepackage{microtype}
\usepackage{algorithm}
\usepackage{algorithmic}

\include{notation_song}

\usepackage[dvipsnames]{xcolor}

\colorlet{mylinkcolor}{NavyBlue}
\colorlet{mycitecolor}{NavyBlue}
\colorlet{myurlcolor}{NavyBlue}

\@ifundefined{showcaptionsetup}{}{%
	\PassOptionsToPackage{caption=false}{subfig}}

\providecommand{\assumptionname}{Assumption}
\providecommand{\lemmaname}{Lemma}
\providecommand{\propositionname}{Proposition}
\providecommand{\corollaryname}{Corollary}
\providecommand{\theoremname}{Theorem}

\@ifundefined{showcaptionsetup}{}{%
	\PassOptionsToPackage{caption=false}{subfig}}
\usepackage{subfig}
\makeatother

\providecommand{\assumptionname}{Assumption}
\providecommand{\lemmaname}{Lemma}
\providecommand{\propositionname}{Proposition}
\providecommand{\corollaryname}{Corollary}
\providecommand{\theoremname}{Theorem}
\begin{document}
\title{Trimmed Density Ratio Estimation}
\author{Song Liu \thanks{This work was done when Song Liu was at The Institute of Statistical Mathematics, Japan} \\
	song.liu@bristol.ac.uk, \\
	University of Bristol\\
	\\
	Akiko Takeda \\
	atakeda@ism.ac.jp, \\
	The Institute of Statistical Mathematics, \\
	Center for Advanced Intelligence Project (AIP), RIKEN\\
	\\
	Taiji Suzuki \\
	suzuki.t.ct@m.titech.ac.jp, \\
	Tokyo Institute of Technology, \\
	Sakigake (PRESTO), JST \\
	Center for Advanced Intelligence Project (AIP), RIKEN \\
	\\
	Kenji Fukumizu \\
	fukumizu@ism.ac.jp\\
	The Institute of Statistical Mathematics}
\maketitle
\begin{abstract}
Density ratio estimation is a vital tool in both machine learning and statistical community. However, due to the unbounded nature of density ratio, the estimation procedure can be vulnerable to corrupted data points, which often
pushes the estimated ratio toward infinity. In this paper, we present
a robust estimator which automatically identifies and trims outliers.
The proposed estimator has a convex formulation, and the global optimum
can be obtained via subgradient descent. We analyze the parameter
estimation error of this estimator under high-dimensional settings.
Experiments are conducted to verify
the effectiveness of the estimator. 
\end{abstract}

\section{Introduction}
\label{sec.intro}
Density ratio estimation (DRE) \cite{Nguyen2010, Huang2007,Sugiyama2012} is an important tool
in various branches of machine learning and statistics.
Due to its ability of directly modelling the differences between two
probability density functions, DRE finds its applications in change
detection \cite{Kawahara2012,Fazayeli2016}, two-sample test
\cite{Wornowizki2016} and outlier detection \cite{azmandian2012local, Smola2009}. In
recent years, a sampling framework called Generative Adversarial Network
(GAN) (see e.g., \cite{Goodfellow2014,Nowozin2016}) uses the density
ratio function to compare artificial samples from a generative distribution and real samples from
an unknown distribution. DRE has also been widely discussed in statistical literatures for adjusting non-parametric density estimation \cite{Efron1996}, stabilizing the estimation of heavy tailed distribution \cite{Fithian2015} and fitting multiple distributions at once \cite{Fokianos2004}.

However, as a density ratio function can grow unbounded,
DRE can suffer from robustness and stability issues: 
a few corrupted points may completely mislead the
estimator (see Figure \ref{fig.mnchange} in  Section \ref{sec.label} for example). Considering a density ratio $p(x)/q(x)$, a point $x$ that is extremely far away from the high density region of $q$ may have an almost infinite ratio value and DRE results can be \emph{dominated} by such points.   
This makes DRE performance very sensitive to rare pathological data or small modifications of the dataset. Here we give two examples:
\paragraph{Cyber-attack}
In change detection applications, a density ratio $p(x)/q(x)$ is used to determine how the data generating model differs between $p$ and $q$. 
Consider a ``hacker''
who can spy on our data may just inject a few data points in $p$
which are extremely far away from the high-density region of $q$. This would result excessively large $p(x)/q(x)$ tricking us to believe there is a
significant change from $q(x)$ to $p(x)$, even if there is no change at all. If the generated outliers are also far away from the high density region of $p(x)$, we end up with a very different density ratio function and the original parametric pattern in the ratio is ruined. We give such an example in Section \ref{sec.label}.
%


\paragraph{Volatile Samples}
The change of external environment may be responded in unpredictable ways.
It is possible that a small portion of samples react more ``aggressively'' to the change than the others. These samples may be skewed and show very high density ratios, even if the change of distribution is relatively mild when these volatile samples are excluded. 
For example, when testing a new fertilizer, a small number of plants may fail to adapt, even if the vast majority of crops are healthy. 


Overly large density ratio values can cause further troubles when the ratio is used to weight samples. For example, in the domain adaptation setting, we may reweight samples from one task and reuse them in another task. Density ratio is a natural choice of such ``importance weighting'' scheme \cite{Sugiyama2008,Shimodaira2000}. 
However, if one or a few samples have extremely high ratio, after renormalizing, other samples will have almost zero weights and have little impact to the learning task. 


Several methods have been proposed to solve this problem. The relative
density ratio estimation \cite{Yamada2013} estimates a ``biased'' version of density ratio
controlled by a mixture parameter $\alpha.$ The relative density
ratio is always upper-bounded by $\frac{1}{\alpha}$, 
which can give a more robust estimator. 
However,
it is not clear how to de-bias such an estimator
to recover the true density ratio function.
\cite{Smola2009} took a more direct approach. It estimates a \emph{thresholded}
density ratio by setting up a tolerance $t$ to the density
ratio value. All likelihood ratio values bigger than $t$
will be clipped to $t$. The estimator was derived from Fenchel duality
for $f$-divergence \cite{Nguyen2010}. However, 
the optimization for the estimator is not convex
if one uses log-linear models.
The formulation also relies on the non-parametric approximation of the
density ratio function (or the log ratio function) 
making the learned model
hard to interpret. Moreover, there is no intuitive way to directly control the proportion of ratios that are thresholded.
Nonetheless, the concept studied in our paper
is inspired by this pioneering work.

In this paper, we propose a novel method based on a ``trimmed Maximum
Likelihood Estimator'' \cite{Neykov1990, Hadi1997}. This idea relies on a specific
type of density ratio estimator (called log-linear KLIEP) \cite{Tsuboi2009}
which can be written as a maximum likelihood formulation. We simply
``ignore'' samples that make the empirical likelihood take exceedingly
large values. The trimmed density ratio estimator can be formulated as a convex optimization
and translated into a weighted M-estimator. This helps us develop
a simple subgradient-based algorithm that is guaranteed to reach the
global optimum.

Moreover, we shall prove that in addition to recovering the correct density
ratio under the outlier setting, the estimator can also obtain a ``corrected'' 
density ratio function under a truncation setting. It ignores
``pathological'' samples and recovers density ratio only using ``healthy'' samples. 

Although trimming will usually result a more robust estimate of the density ratio function, we also point out that it should not be abused. For example, in the tasks of two-sample test, a diverging density ratio might indicate interesting structural differences between two distributions. 

In Section \ref{sec:Preliminary:-Trimmed-Maximum},
we explain some preliminaries on trimmed maximum likelihood estimator.
In Section \ref{sec:Trimmed-Density-Ratio}, we introduce a trimmed
DRE.
We solve it using a convex formulation whose optimization
procedure is explained in Section \ref{sec:Optimization}. In Section
\ref{sec:-Consistency-in}, we prove the estimation error
upper-bound with respect to a sparsity inducing regularizer. Finally, experimental results are shown in Section \ref{sec.label} and we conclude our work in Section \ref{sec.concl}.

\section{Preliminary: Trimmed Maximum Likelihood Estimation\label{sec:Preliminary:-Trimmed-Maximum}}
Although our main purpose is to estimate the density ratio, we first
introduce the basic concept of \emph{trimmed estimator} using density
functions as examples. Given $n$ samples drawn from a distribution
$P$, i.e., $X:=\left\{ \boldsymbol{x}^{(i)}\right\} _{i=1}^{n}\iid P,\boldsymbol{x}\in\mathbb{R}^{d}$,
we want to estimate the density function $p(\boldsymbol{x})$. Suppose
the true density function is a member of \emph{exponential family}
\cite{Pitman1936}, 
\begin{align}
p(\boldsymbol{x};\boldsymbol{\theta})=\exp\left[\langle\boldsymbol{\theta},\boldsymbol{f}(\boldsymbol{x})\rangle-\log Z(\boldsymbol{\theta})\right],~~Z(\boldsymbol{\theta})=\int q(\boldsymbol{x})\exp\langle\boldsymbol{\theta},\boldsymbol{f}(\boldsymbol{x})\rangle d\boldsymbol{x}\label{eq:expfamily}
\end{align}
where $\boldsymbol{f}(\boldsymbol{x})$ is the sufficient statistics,
$Z(\boldsymbol{\theta})$ is the normalization function and $q(\boldsymbol{x})$
is the base measure.

Maximum Likelihood Estimator (MLE) maximizes the
empirical likelihood over the entire dataset. In contrast, a \emph{trimmed
} MLE only maximizes the likelihood over a \emph{subset} of samples
according to their likelihood values (see e.g., \cite{Hadi1997,Vandev1998}).
This paradigm can be used to derive a popular outlier detection method, one-class Support Vector Machine (one-SVM) \cite{Schoelkopf2000}. 
The derivation is crucial to the development of our trimmed density ratio estimator in later sections. 


Without loss of generality, we can
set the log likelihood function as $\log p(\boldsymbol{x}^{(i)};\boldsymbol{\theta})-\tau_{0},$
where $\tau_{0}$ is a constant. 
As samples corresponding to high likelihood values are likely to be inliers, we can trim all samples whose likelihood is bigger than $\tau_{0}$
using a clipping function $[\cdot]_{-}$, i.e., 
$
\hat{\boldtheta}=\arg\max_{\boldtheta}\sum_{i=1}^{n}[\log p(\boldsymbol{x}^{(i)};\boldsymbol{\theta})-\tau_{0}]_{-},\label{eq:trimMLE}
$
where $[\ell]_{-}$ returns $\ell$ if $\ell\le0$ and $0$ otherwise.\emph{
}This optimization has a \emph{convex }formulation: 
\begin{align}
\min_{\boldsymbol{\theta},\boldsymbol{\epsilon}\ge0}\langle\boldsymbol{\epsilon},\boldsymbol{1}\rangle,~~\text{ s.t. }\forall i,\log p\left(\boldsymbol{x}^{(i)};\boldsymbol{\theta}\right)\ge\tau_{0}-\epsilon_{i},\label{eq:trimmedMLE}
\end{align}
where $\boldsymbol{\epsilon}$ is the slack variable measuring the
difference between $\log p\left(\boldsymbol{x}^{(i)};\boldsymbol{\theta}\right)$
and $\tau_{0}$. However, formulation \eqref{eq:trimmedMLE} is not
practical since computing the normalization term $Z(\boldsymbol{\theta})$
in \eqref{eq:expfamily} is intractable for a general \textbf{$\boldsymbol{f}$}
and it is unclear how to set the trimming level $\tau_{0}$. Therefore we ignore the normalization term and introduce other control
terms: 
\begin{align}
\min_{\boldsymbol{\theta},\boldsymbol{\epsilon}\ge0,\tau\ge0}\frac{1}{2}\|\boldsymbol{\theta}\|^{2}-\nu\tau+\frac{1}{n}\langle\boldsymbol{\epsilon},\boldsymbol{1}\rangle~~\text{ s.t. }\forall i,\langle\boldsymbol{\theta},\boldsymbol{f}(\boldsymbol{x}^{(i)})\rangle\ge\tau-\epsilon_{i}.\label{eq:onesvm}
\end{align}
The $\ell_{2}$ regularization term is introduced to avoid $\boldsymbol{\theta}$
reaching unbounded values.
 A new hyper parameter $\nu\in(0,1]$ replaces $\tau_{0}$ to control
the number of trimmed samples. It can be proven using KKT conditions that at most $1-\nu$
fraction of samples are discarded (see e.g., \cite{Schoelkopf2000}, Proposition 1 for details). 
Now we have reached the standard formulation
of one-SVM. 

This trimmed estimator ignores the large likelihood
values and creates a focus only on the low density region. Such a
trimming strategy allows us to discover ``novel'' points or outliers
which are usually far away from the high density area.

\section{Trimmed Density Ratio Estimation\label{sec:Trimmed-Density-Ratio}}

In this paper, our main focus is to derive a \emph{robust} density
ratio estimator following a similar trimming strategy. First, we briefly
review the a density ratio estimator \cite{Sugiyama2012} from the
perspective of Kullback-Leibler divergence minimization.

\subsection{Density Ratio Estimation (DRE)}

For two sets of data 
$
X_{p}:=\{\boldsymbol{x}_{p}^{(1)},\dots,\boldsymbol{x}_{p}^{(n_{p})}\}\iid P,X_{q}:=\{\boldsymbol{x}_{q}^{(1)},\dots,\boldsymbol{x}_{q}^{(n_{q})}\}\iid Q,
$
 assume both the densities $p(\boldsymbol{x})$ and $q(\boldsymbol{x})$
are in exponential family \eqref{eq:expfamily}. We know 
$
\frac{p(\boldsymbol{x};\boldsymbol{\theta}_{p})}{q(\boldsymbol{x};\boldsymbol{\theta}_{q})}\propto\exp\left[\langle\boldsymbol{\theta}_{p}-\boldsymbol{\theta}_{q},\boldsymbol{f}(\boldsymbol{x})\rangle\right].
$
Observing that the data $\boldsymbol{x}$ only interacts with the
parameter $\boldsymbol{\theta}_{p}-\boldsymbol{\theta}_{q}$ through
$\boldsymbol{f}$ , we can keep using $\boldsymbol{f}(\boldsymbol{x})$
as our sufficient statistic for the density ratio model, and merge
two parameters $\boldsymbol{\theta}_{p}$ and $\boldsymbol{\theta}_{q}$
into one single parameter $\boldsymbol{\delta}=\boldsymbol{\theta}_{p}-\boldsymbol{\theta}_{q}$.
Now we can model our density ratio as 
\begin{align}
r(\boldsymbol{x};\boldsymbol{\delta}):=\exp\left[\langle\boldsymbol{\delta},\boldsymbol{f}(\boldsymbol{x})\rangle-\log N(\boldsymbol{\delta})\right],~N(\boldsymbol{\delta}):=\int q(\boldsymbol{x})\exp\langle\boldsymbol{\delta},\boldsymbol{f}(\boldsymbol{x})\rangle d\boldsymbol{x},\label{eq:ratio_model}
\end{align}
where $N(\boldsymbol{\delta})$ is the normalization term that guarantees
$\int q(\boldsymbol{x})r(\boldsymbol{x};\boldsymbol{\delta})d\boldsymbol{x}=1$
so that $q(\boldsymbol{x})r(\boldsymbol{x};\boldsymbol{\delta})$
is a valid density function and is normalized over its domain.

Interestingly, despite the parameterization (changing from $\boldsymbol{\theta}$
to $\boldsymbol{\delta}$), \eqref{eq:ratio_model} is exactly the
same as \eqref{eq:expfamily} where $q(\boldsymbol{x})$ appeared
as a base measure. The difference is, here, $q(\boldsymbol{x})$ is
a \emph{density function} from which $X_{q}$ are drawn so that $N(\boldsymbol{\delta})$
can be approximated accurately from samples of $Q$. Let us define
\begin{align}
\label{eq:ratiomodel2}
\hat{r}(\boldsymbol{x};\boldsymbol{\delta}):=\exp\left[\langle\boldsymbol{\delta},\boldsymbol{f}(\boldsymbol{x})\rangle-\log\widehat{N}(\boldsymbol{\delta})\right],~\widehat{N}(\boldsymbol{\delta}):=\frac{1}{n_{q}}\sum_{j=1}^{n_{q}}\exp\left[\langle\boldsymbol{\delta},\boldsymbol{f}(\boldsymbol{x}_{q}^{(j)})\rangle\right].
\end{align}
Note this model can be computed for any $\boldsymbol{f}$ even if
the integral in $N(\boldsymbol{\delta})$ does not have a closed form
.

In order to estimate $\boldsymbol{\delta}$, we minimize the Kullback-Leibler
divergence between $p$ and $q\cdot r_{\boldsymbol{\delta}}$: 
\begin{align}
\min_{\boldsymbol{\delta}}\mathrm{KL}\left[p|q\cdot r_{\boldsymbol{\delta}}\right] & =\min_{\boldsymbol{\delta}}\int p(\boldsymbol{x})\log\frac{p(\boldsymbol{x})}{q(\boldsymbol{x})r(\boldsymbol{x};\boldsymbol{\delta})}d\boldsymbol{x}=c-\max_{\boldsymbol{\delta}}\int p(\boldsymbol{x})\log r(\boldsymbol{x};\boldsymbol{\delta})d\boldsymbol{x}\nonumber \\
 & \approx c-\max_{\boldsymbol{\delta}}\frac{1}{n_{p}}\sum_{i=1}^{n_{p}}\log\hat{r}(\boldsymbol{x}_{p}^{(i)};\boldsymbol{\delta})\label{eq:MLEOBJ}
\end{align}
where $c$ is a constant irrelevant to $\boldsymbol{\delta}$. It
can be seen that the minimization of KL divergence boils down to \emph{maximizing
log likelihood} \emph{ratio} over dataset $X_{p}$.

Now we have reached the log-linear Kullback-Leibler Importance Estimation
Procedure (log-linear KLIEP) estimator \cite{Tsuboi2009,Liu2016a}. 



%

\subsection{Trimmed Maximum Likelihood Ratio}
As stated in Section \ref{sec.intro}, to rule out the influences of large density ratio, we trim samples with large likelihood
ratio values from \eqref{eq:MLEOBJ}. Similarly to one-SVM in  \eqref{eq:trimmedMLE},
we can consider a trimmed MLE 
$
\hat{\boldsymbol{\delta}}=\arg\max_{\bolddelta}\sum_{i=1}^{n_{p}}[\log\hat{r}(\boldsymbol{x}_{p}^{(i)};\boldsymbol{\delta})-t_{0}]_{-}\label{eq:hingeMLE}
$
where $t_{0}$ is a threshold above which the likelihood ratios are
ignored. It has a convex formulation: 
\begin{align}
\min_{\boldsymbol{\delta},\boldsymbol{\epsilon}\ge\boldsymbol{0}}\langle\epsilon,\boldsymbol{1}\rangle,~~\text{s.t. }\forall\boldsymbol{x}_{p}^{(i)}\in X_{p},\log\hat{r}(\boldsymbol{x}_{p}^{(i)};\boldsymbol{\delta})\ge t_{0}-\epsilon_{i}.\label{eq:objsimple}
\end{align}
\eqref{eq:objsimple} is similar to \eqref{eq:trimmedMLE} since we
have only replaced $p(\boldsymbol{x};\boldsymbol{\theta})$ with $\hat{r}(\boldsymbol{x};\boldsymbol{\delta})$.
However, the ratio model $\hat{r}(\boldsymbol{x};\boldsymbol{\delta})$
in \eqref{eq:objsimple} comes with a tractable normalization term
$\hat{N}$ while the normalization term $Z$ in $p(\boldsymbol{x};\boldsymbol{\theta})$
is in general intractable.

Similar to \eqref{eq:onesvm}, we can directly control the trimming
quantile via a hyper-parameter $\nu$: 
\begin{align}
\min_{\boldsymbol{\delta},\boldsymbol{\epsilon}\ge\boldsymbol{0},t\ge0}\frac{1}{n_{p}}\langle\epsilon,\boldsymbol{1}\rangle-\nu\cdot t+\lambda R(\boldsymbol{\delta}),~~\text{s.t. }\forall\boldsymbol{x}_{p}^{(i)}\text{\ensuremath{\in}}X_{p},\log\hat{r}(\boldsymbol{x}_{p}^{(i)};\boldsymbol{\delta})\ge t-\epsilon_{i}\label{eq:objquantile}
\end{align}
where $R(\boldsymbol{\delta})$ is a convex regularizer. \eqref{eq:objquantile}
is\emph{ }also convex, but it has $n_{p}$ number of \emph{non-linear}
constraints and the search for the global optimal solution can be time-consuming.
To avoid such a problem, one could derive and solve the dual problem
of \eqref{eq:objquantile}. In some applications,
we rely on the primal parameter structure (such as sparsity) for model
interpretation, and feature engineering. In Section \ref{sec:Optimization}, we translate \eqref{eq:objquantile} into an
equivalent form so that its solution is obtained via a subgradient
ascent method which is guaranteed to converge to the global optimum.


One common way to construct a convex robust estimator is using a Huber loss \cite{Huber1964}.
Although the proposed trimming technique rises from a different setting, it shares the same guiding principle with Huber loss: avoid assigning dominating values to outlier likelihoods in the objective function. 

In Section \ref{sec:poofproposvm} in the supplementary material, we show the relationship between trimmed DRE and binary Support Vector Machines \cite{Scholkopf2001,Cristianini2000}.

\section{Optimization\label{sec:Optimization}}
The key to solving \eqref{eq:objquantile} efficiently is reformulating
it into an equivalent $\max\min$ problem. 
\begin{prop}
\label{prop:minma}Assuming $\nu$ is chosen such that $\hat{t}>0$
for all optimal solutions in \eqref{eq:objquantile}, then $\hat{\boldsymbol{\delta}}$
is an optimal solution of \eqref{eq:objquantile} if and only if it
is also the optimal solution of the following $\max \min$ problem: 
\begin{align}
\max_{\boldsymbol{\delta}}\min_{\boldsymbol{w}\in\left[0,\frac{1}{n_{p}}\right]^{n_{p}},\langle\boldsymbol{1},\boldsymbol{w}\rangle=\nu}\mathcal{L}(\boldsymbol{\delta},\boldsymbol{w})-\lambda R(\boldsymbol{\delta}),\; & \mathcal{L}(\boldsymbol{\delta},\boldsymbol{w}):=\sum_{i=1}^{n_{p}}w_{i}\cdot\log\hat{r}(\boldsymbol{x}_{p}^{(i)};\boldsymbol{\delta}).\label{eq:min_max}
\end{align}
\end{prop}
The proof is in Section \ref{subsec:minma} in the supplementary material.  We define $(\hat{\boldsymbol{\delta}},\hat{\boldsymbol{w}})$ as a
saddle point of \eqref{eq:min_max}:
\begin{align}
\label{eq:saddlepoint}
\nabla_{\boldsymbol{\delta}}\mathcal{L}(\hat{\boldsymbol{\delta}},\hat{\boldsymbol{w}})-\nabla_{\boldsymbol{\delta}}\lambda R(\hat{\boldsymbol{\delta}})=\boldzero,\hat{\boldsymbol{w}}\in\arg\min_{\boldsymbol{w}\in[0,\frac{1}{n_p}]^{n_{p}},\langle\boldsymbol{w},\boldsymbol{1}\rangle=\nu}\mathcal{L}(\hat{\boldsymbol{\delta}},\boldsymbol{w}),
\end{align}
where the second $\nabla_{\boldsymbol{\delta}}$ means the subgradient
if $R$ is sub-differentiable. 

Now the ``trimming'' process of our estimator can be clearly seen
from \eqref{eq:min_max}: The $\max$ procedure estimates a density
ratio given the currently assigned weights $\boldsymbol{w}$, and
the $\min$ procedure trims the large log likelihood ratio values
by assigning corresponding $w_{i}$ to $0$ (or values smaller
than $\frac{1}{n_{p}}$). 
For simplicity, we only consider the cases where $\nu$ is a multiple of $\frac{1}{n_p}$. 
Intuitively, $1-\nu$ is the proportion of likelihood ratios that are trimmed thus $\nu$ should not be greater than 1. 
Note if we set $\nu=1$, \eqref{eq:min_max}
is equivalent to the standard density ratio estimator \eqref{eq:MLEOBJ}. 
Downweighting outliers while estimating the model parameter $\boldsymbol{\delta}$
is commonly used by robust estimators (See e.g., \cite{Cleveland1979,Suykens2002}). 

The search for $(\hat{\boldsymbol{\delta}},\hat{\boldsymbol{w}})$
is straightforward. It is easy to solve with respect to $\boldsymbol{w}$
or $\boldsymbol{\delta}$ while the other is fixed: given a parameter
$\boldsymbol{\delta}$, the optimization with respect to $\boldsymbol{w}$
is a linear programming and \emph{one of} the extreme optimal solutions is
attained by assigning weight $\frac{1}{n_{p}}$
to the elements that correspond to the $\nu n_{p}$-smallest log-likelihood
ratio $\log\hat{r}(\boldsymbol{x}^{(i)},\boldsymbol{\delta})$. This
observation leads to a simple ``gradient ascent and trimming'' algorithm
(see Algorithm \ref{alg}). In Algorithm \ref{alg}, 
\[
\nabla_{\boldsymbol{\delta}}\mathcal{L}(\boldsymbol{\delta},\boldsymbol{w})=\frac{1}{n_{p}}\sum_{i=1}^{n_{p}}w_{i}\cdot\boldsymbol{f}(\boldsymbol{x}_{p}^{(i)})-\nu\cdot\sum_{j=1}^{n_{q}}\frac{e^{(j)}}{\sum_{k=1}^{n_{q}}e^{(k)}}\boldsymbol{f}(x_{q}^{(j)}),~~e^{(i)}:=\exp(\langle\boldsymbol{\delta},\boldsymbol{f}(x_{q}^{(i)})\rangle).\label{eq:subgra}
\]

\begin{algorithm}[t]
\begin{algorithmic}

\STATE Input: $X_{p},X_{q},\nu$ and step sizes $\{\eta_{\mathrm{it}}\}_{\mathrm{it=1}}^{\mathrm{it}_\mathrm{max}}$;
Initialize $\boldsymbol{\delta}_{0},\boldsymbol{w}_{0}$, Iteration
counter: $\mathrm{it}=0$, Maximum number of iterations: $\mathrm{it}_{\mathrm{max}}$,
Best objective, parameter pair $(O_{\mathrm{best}}=-\infty,\boldsymbol{\delta}_{\mathrm{\mathrm{best}}},\boldsymbol{w}_{\mathrm{best}})$
. 

\WHILE{not converged and $\mathrm{it}\le\mathrm{it}_{\mathrm{max}}$}

\STATE Obtain a sorted set $\left\{ \boldsymbol{x}_{p}^{(i)}\right\} _{i=1}^{n_{p}}$
so that $\log\hat{r}(\boldsymbol{x}_{p}^{(1)};\boldsymbol{\delta}_{\mathrm{it}})\le\log\hat{r}(\boldsymbol{x}_{p}^{(2)};\boldsymbol{\delta}_{\mathrm{it}})\cdots\le\log\hat{r}(\boldsymbol{x}_{p}^{(n_{p})};\boldsymbol{\delta}_{\mathrm{it}}).$

\STATE $w_{\mathrm{it}+1,i}=\frac{1}{n_{p}},\forall i\le\nu n_{p}.~~w_{\mathrm{it+1},i}=0,\text{otherwise.}$

\STATE Gradient ascent with respect to $\boldsymbol{\delta}$: $\boldsymbol{\delta}_{\mathrm{it}+1}  =\boldsymbol{\delta}_{\mathrm{it}}+\eta_{\mathrm{it}}\cdot\nabla_{\boldsymbol{\delta}}[\mathcal{L}(\boldsymbol{\delta}_{\mathrm{it}},\boldsymbol{w}_{\mathrm{it}+1})-\lambda R(\boldsymbol{\delta}_{\mathrm{it}})],\:$
\STATE 
$
O_{\mathrm{best}}  =\max(O_{\mathrm{best}},\mathcal{L}(\boldsymbol{\delta}_{\mathrm{it+1}},\boldsymbol{w}_{\mathrm{it}+1}))
$
and update $(\bolddelta_\mathrm{best}, \boldw_\mathrm{best})$ accordingly. 
$~~\mathrm{it}=\mathrm{it}+1.$

\ENDWHILE

Output: $(\boldsymbol{\delta}_{\mathrm{best}},\boldsymbol{w}_{\mathrm{best}})$
\end{algorithmic} 
\caption{Gradient Ascent and Trimming \label{alg}}
\end{algorithm}

In fact, Algorithm \ref{alg} is a subgradient method \citep{Boyd2014,Nedic2009},
since the optimal value function of the inner problem of \eqref{eq:min_max} is not
differentiable at some $\boldsymbol{\delta}$ where the inner problem
has multiple optimal solutions. The subdifferential of the optimal value of the inner problem with respect to $\bolddelta$ can be a \emph{set} but Algorithm \ref{alg} only computes a subgradient
obtained using the extreme point solution $\boldsymbol{w}_{\mathrm{it+1}}$
of the inner linear programming. Under mild conditions, this subgradient ascent approach will converge to
optimal results with diminishing step size rule and $\mathrm{it}\rightarrow\infty$.
See \citep{Boyd2014} for details. 

Algorithm \ref{alg} is a simple gradient ascent procedure and can be implemented by deep learning softwares such as Tensorflow\footnote{\url{https://www.tensorflow.org/}} which benefits from the GPU acceleration. In contrast, the original problem \eqref{eq:objquantile}, due to its heavily constrained nature, cannot be easily programmed using such a framework. 

%
%

\section{Estimation Consistency in High-dimensional Settings \label{sec:-Consistency-in}}
\begin{figure}
	\centering
	\subfloat[Outlier Setting. Blue and red points are i.i.d.]{
		\includegraphics[width=.49\textwidth]{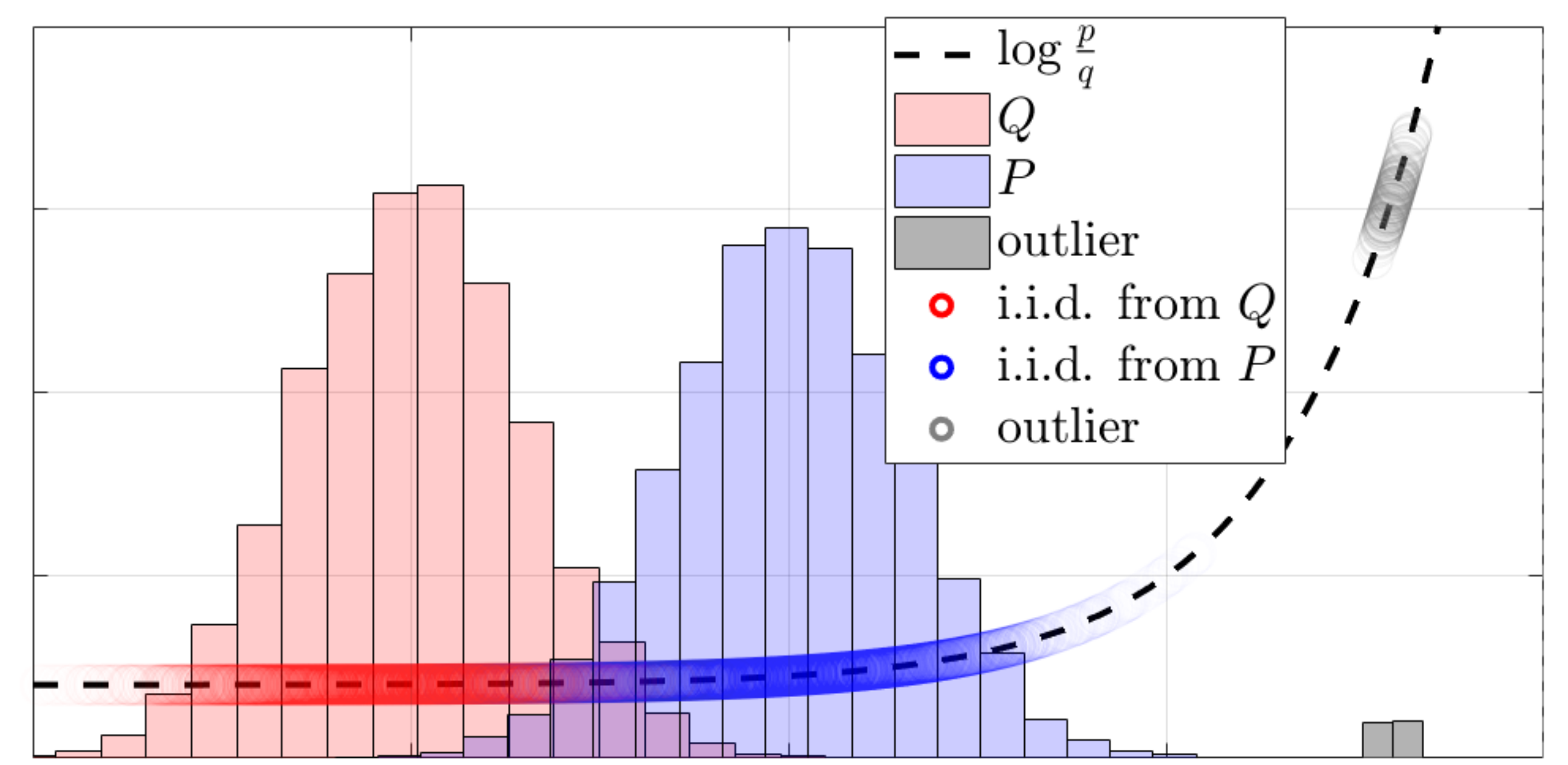}
		\label{fig:illus1}
	}
	\subfloat[Truncation Setting. There are no outliers.]{
		\label{fig:illus2}
	\includegraphics[width=.49\textwidth]{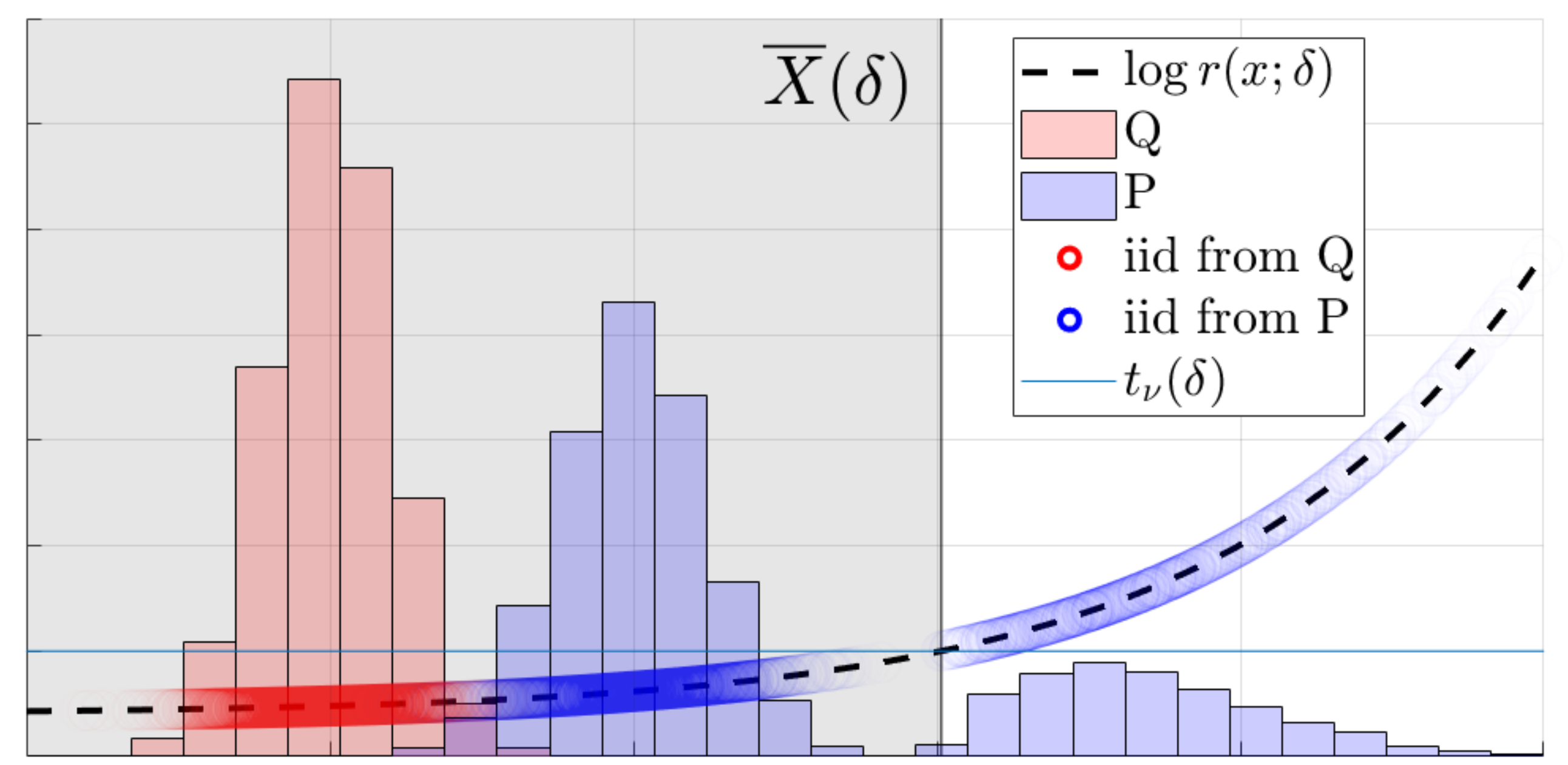}
	}
\caption{Two settings of theoretical analysis.}
\end{figure}
In this section, we show how the estimated parameter $\hat{\boldsymbol{\delta}}$ in \eqref{eq:saddlepoint} converges to the ``optimal parameters''
$\boldsymbol{\delta}^{*}$ 
as both sample size and dimensionality goes to infinity
under the ``outlier'' and ``truncation'' setting respectively.  

In the \textbf{outlier setting} (Figure \ref{fig:illus1}),
we assume $X_p$ is contaminated by outliers and all ``inlier'' samples in $X_p$ are i.i.d.. The  outliers are injected into our dataset $X_p$ after looking at our inliers. For example, hackers can spy on our data and inject fake samples so that our estimator exaggerates the degree of change.

In the \textbf{truncation setting}, there are no outliers. $X_p$ and $X_q$ are i.i.d. samples from $P$ and $Q$ respectively. However, we have a subset of ``volatile'' samples in $X_p$ (the rightmost mode on  histogram in Figure \ref{fig:illus2}) that are pathological and exhibit large density ratio values.

In the theoretical results in this section, we focus on analyzing the performance of our estimator for high-dimensional data assuming the number of non-zero elements in the optimal $\bolddelta^*$ is $k$ and use the $\ell_1$ regularizer, i.e., $R(\boldsymbol{\theta})=\|\boldsymbol{\theta}\|_{1}$ which induces sparsity on $\hat{\bolddelta}$.
The proofs
rely on a recent development \cite{Yang2015,yang2016high}
where a ``weighted'' high-dimensional estimator was studied. 
We also assume the optimization of $\bolddelta$ in \eqref{eq:min_max} was conducted
within an $\ell_{1}$ ball of width $\rho$, i.e., $\mathrm{Ball}(\rho)$,
and $\rho$ is wisely chosen so that the optimal parameter $\boldsymbol{\delta}^{*}\in\mathrm{Ball}(\rho)$.
The same technique was used in previous works \cite{Loh2015,Yang2015}.

\paragraph{Notations:}
We denote $\boldw^* \in \mathbb{R}^{n_p}$ as the ``optimal'' weights depending on $\bolddelta^*$ and our data. 
To lighten the notation, we shorten the \emph{log} density ratio model as $z_{\boldsymbol{\delta}}(\boldsymbol{x}):=\log r(\boldsymbol{x};\boldsymbol{\delta}), \hat{z}{}_{\boldsymbol{\delta}}(\boldsymbol{x}):=\log\hat{r}(\boldsymbol{x};\boldsymbol{\delta})$

The proof of Theorem \ref{thm:main}, \ref{thm:main-1} and \ref{thm:main-1-1} can be found in Section \ref{subsec:theorem1}, \ref{sec:proofcol1} and \ref{sec:proofoftruncation} in supplementary materials.

\subsection{A Base Theorem\label{subsec:Preparations}}
Now we provide a base theorem giving
an upperbound of $\|\hat{\bolddelta}-\bolddelta^{*}\|$.
We state this theorem only with
respect to an arbitrary pair $(\bolddelta^{*},\boldw^{*})$ and 
the pair is set properly later in Section \ref{sec:outliersetting} and \ref{sec:truncationsetting}.

We make a few regularity conditions on samples from $Q$. Samples of $X_{q}$ should be well behaved in terms of log-likelihood
ratio values. 
\begin{assumption}
	\label{assu:.boundedrq}$\exists0<c_{1}<1,1<c_{2}<\infty$ 
	$
	\forall\boldsymbol{x}_{q}\in X_{q},\boldsymbol{u}\in\mathrm{Ball}(\rho),c_{1}\le\exp\langle\boldsymbol{\delta}^{*}+\boldsymbol{u},\boldsymbol{x}_{q}\rangle\le c_{2}
	$
	and collectively $c_{2}/c_{1}=C_{r}$. 
\end{assumption}
We also assume the Restricted Strong Convexity (RSC) condition on the covariance of $\boldX_q$, i.e., $\mathrm{cov}(\boldX_q) = \frac{1}{n_q} (\boldX_q - \frac{1}{n_q}\boldX_q \bold1)  (\boldX_q - \frac{1}{n_q}\boldX_q \bold1)^\top$. Note this property has been verified for various different design
matrices $\boldX_{q}$, such as Gaussian or sub-Gaussian (See, e.g., \cite{Raskutti2010,Rudelson2013}).  
\begin{assumption}
	\label{assu:RSC} RSC condition of $\mathrm{cov}(\boldX_q)$
	holds for all $\boldsymbol{u}$, i.e.,
	there exists $\kappa'_{1} > 0$ and $c>0$ such that $\boldsymbol{u}^{\top}\mathrm{cov}(\boldX_q)\boldsymbol{u}\ge\kappa'_{1}\|\boldsymbol{u}\|^{2}-\frac{c}{\sqrt{n_{q}}}\|\boldsymbol{u}\|_{1}^{2}\text{ with high probability. }$
\end{assumption}
\begin{thm}
In addition to Assumption \ref{assu:.boundedrq} and \ref{assu:RSC}, 
there exists coherence between parameter $\boldsymbol{w}$ and
$\boldsymbol{\delta}$ at a saddle point $(\hat{\bolddelta},\hat{\boldw})$:
\begin{align}
\langle\nabla_{\bolddelta}\mathcal{L}(\hat{\bolddelta},\hat{\boldw})-\nabla_{\boldsymbol{\delta}}\mathcal{L}(\hat{\bolddelta},\boldsymbol{w}^{*}),\hat{\boldsymbol{u}}\rangle\ge-\kappa_{2}\|\hat{\boldsymbol{u}}\|^{2}-\tau_{2}(n,d)\|\hat{\boldsymbol{u}}\|_{1},\label{eq:ass2}
\end{align}
where $\hat{\boldu}:=\hat{\bolddelta}-\bolddelta^{*}$, $\kappa_2 > 0$ is a constant and $\tau_2(d,n) > 0$. It can be
shown that if \[\textstyle{
\lambda_{n}\ge2\max\left[\|\nabla_{\boldsymbol{\delta}}\mathcal{L}(\boldsymbol{\delta}^{*},\boldsymbol{w}^{*})\|_{\infty},\frac{\rho\nu c}{2C_{r}^{2}\sqrt{n_{q}}},\tau_{2}(n,d)\right]}
\]
and $\nu\kappa'_{1}>2C_{r}^{2}\kappa_{2}$, where $c>0$ is a constant determined by RSC condition, we are guaranteed that $\|\hat{\boldsymbol{\delta}}-\boldsymbol{\delta}^{*}\|\le\frac{C_{r}^{2}}{(\nu\kappa'_{1}-2C_{r}^{2}\kappa_{2})}\cdot\frac{3\sqrt{k}\lambda_{n}}{2}$ with probability converging to one.
\label{thm:main} 
\end{thm}
The condition \eqref{eq:ass2} states that if we swap $\hat{\boldw}$ for $\boldw^*$, the change of the gradient $\nabla_\bolddelta\mathcal{L}$ is limited. 
Intuitively, it shows that our estimator \eqref{eq:min_max} is not ``picky'' on $\boldw$: even if we cannot have the optimal weight assignment $\boldw^*$, we can still use ``the next best thing'', $\hat{\boldw}$ to compute the gradient which is close enough.
We later show how \eqref{eq:ass2} is satisfied.
Note if $\|\nabla_{\boldsymbol{\delta}}\mathcal{L}(\boldsymbol{\delta}^{*},\boldsymbol{w}^{*})\|_{\infty}$,
$\tau_{2}(n,d)$ converge to zero as $n_{p},n_{q},d\to\infty$,
by taking $\lambda_{n}$ as such, Theorem 1 guarantees the consistency
of $\hat{\boldsymbol{\delta}}$.
In Section \ref{subsec:Consistency-under-Outlier} and \ref{subsec:trimmedsetting},
we explore two different settings of $(\boldsymbol{\delta}^{*},\boldw^{*})$
 that make $||\hat{\boldsymbol{\delta}}-\boldsymbol{\delta}^{*}\|$
converges to zero.

\subsection{Consistency under Outlier Setting\label{subsec:Consistency-under-Outlier}}
\label{sec:outliersetting}
\paragraph{Setting:}
Suppose dataset $X_{p}$ is the union of two disjoint sets $G$ (Good points) and
$B$ (Bad points) such that $G\iid p(\boldx)$ 
and $\min_{j\in B}z_{\boldsymbol{\delta}^{*}}(\boldsymbol{x}_{p}^{(j)})>\max_{i\in G}z_{\boldsymbol{\delta}^{*}}(\boldsymbol{x}_{p}^{(i)})$ (see Figure \ref{fig:illus1}). Dataset $X_{q} \iid q(\boldx)$ does \emph{not
}contain any outlier. We set $\nu=\frac{|G|}{n_{p}}.$ The optimal
parameter $\boldsymbol{\delta}^{*}$ is set such that $p(\boldx)=q(\boldx)r(\boldx;\boldsymbol{\delta}^{*})$. We set
$\boldsymbol{w}_{i}^{*}=
\frac{1}{n_{p}}, \forall \boldsymbol{x}_{p}^{(i)}\text{ \ensuremath{\in}G}$ and 0 otherwise.
\paragraph{Remark:} Knowing the inlier proportion $|G|/n_{p}$ is a strong assumption. However it is only imposed for  theoretical analysis. As we show in Section \ref{sec.label}, our method works well even if $\nu$ is only a rough guess (like $90\%$). Loosening this assumption will be an important future work. 
\begin{assumption}
	\label{ass.ratio.lip-1} 
$
	\forall\boldsymbol{u}\in\mathrm{Ball}(\rho),\sup_{\boldsymbol{x}}\left|\hat{z}{}_{\boldsymbol{\delta}^{*}+\boldsymbol{u}}(\boldsymbol{x})-\hat{z}_{\boldsymbol{\delta}^{*}}(\boldsymbol{x})\right|\le C_{\mathrm{lip}}\|\boldsymbol{u}\|_{1}.
$
\end{assumption}
This assumption says that the log density ratio model is Lipschitz continuous
around its optimal parameter $\bolddelta^{*}$ and hence there is a limit
how much a log ratio model can deviate from the optimal model under a small perturbation $\boldu$.
As our estimated weights $\hat{w}_i$ depends on the relative ranking of $\hat{z}_{\hat{\bolddelta}} (\boldx_{p}^{(i)})$, this assumption implies that the relative ranking between two points will remain unchanged under a small perturbation $\boldu$ if they are far apart. 
The following theorem shows that if we have enough clearance between ``good''and ``bad samples'', $\hat{\bolddelta}$ converges to the optimal parameter $\bolddelta^*$. 
\begin{thm}
\label{cor:outlier}
In addition to Assumption \ref{assu:.boundedrq}, \ref{assu:RSC} and
a few mild technical conditions (see Section \ref{sec:proofcol1} in the supplementary material), Assumptions \ref{ass.ratio.lip-1} holds. Suppose $\min_{j\in B}z_{\boldsymbol{\delta}^{*}}(\boldsymbol{x}_{p}^{(j)})-\max_{i\in G}z_{\boldsymbol{\delta}^{*}}(\boldsymbol{x}_{p}^{(i)})\ge3C_{\mathrm{lip}}\rho,$
$\nu=\frac{|G|}{n_{p}}, n_q = \Omega(|G|^2)$.
If 
$
\lambda_{n}\ge2\cdot\max\left(\sqrt{\frac{K_{1}\log d}{|G|}},\frac{\rho \nu c}{2C_{r}^{2}\sqrt{n_{q}}}\right),
$
where $K_{1}>0,c>0$ are constants, we are
guaranteed that $||\hat{\boldsymbol{\delta}}-\boldsymbol{\delta}^{*}\|\le\frac{C_{r}^{2}}{\nu\kappa'_{1}}\cdot3\sqrt{k}\lambda_{n}$
with probability converging to 1. \label{thm:main-1} 
\end{thm}

It can be seen that $\|\hat{\boldsymbol{\delta}}-\boldsymbol{\delta}^{*}\| =O\left(\sqrt{{\log d}/{\min(|G|, n_q)}}\right)$ if $d$ is reasonably large. 


\subsection{Consistency under Truncation Setting\label{subsec:trimmedsetting}}
\label{sec:truncationsetting}
In this setting, we do not assume there are outliers in the observed data.
Instead, we examine
the ability of our estimator recovering the density ratio up to a
certain quantile of our data. This ability is especially useful when the
behavior of the tail quantile is more volatile and makes the standard estimator \eqref{eq:MLEOBJ} output
unpredictable results.

\paragraph{Notations:}
Given $\nu \in (0,1]$, we call $t_{\nu}(\boldsymbol{\delta})$ is the $\nu$-th quantile
of $z_{\boldsymbol{\delta}}$ if $P\left[z_{\boldsymbol{\delta}}<t_{\nu}(\boldsymbol{\delta}))\right]\le\nu$
and $P\left[z_{\boldsymbol{\delta}}\le t_{\nu}(\boldsymbol{\delta}))\right]\ge\nu.$
In this setting, we consider $\nu$ is fixed by a user thus we drop
the subscript $\nu$ from all subsequent discussions. Let's define
a truncated domain: $\overline{X}(\boldsymbol{\delta})=\left\{ \boldsymbol{x}\in\mathbb{R}^{d}|z_{\boldsymbol{\delta}}(\boldsymbol{x})<t(\boldsymbol{\delta})\right\} $,
$\overline{X}^{p}(\boldsymbol{\delta})=X_{p}\cap\overline{X}(\boldsymbol{\delta})$
and $\overline{X}^{q}(\boldsymbol{\delta})=X_{q}\cap\overline{X}(\boldsymbol{\delta})$. See Figure \ref{fig:illus2} for a visualization of $t(\bolddelta)$ and $\overline{X}(\bolddelta)$ (the dark shaded region). 

\paragraph{Setting:}
Suppose dataset $X_{p}\iid P$ and $X_{q}\iid Q$. Truncated densities
$\overline{p}_{\boldsymbol{\delta}}$ and $\overline{q}_{\boldsymbol{\delta}}$
are the unbounded densities $p$ and $q$ restricted only on the truncated domain
$\overline{X}(\boldsymbol{\delta})$. Note that the truncated densities
are dependent on the parameter $\boldsymbol{\delta}$ and $\nu$. 
We show that under some assumptions, the parameter $\hat{\boldsymbol{\delta}}$
obtained from \eqref{eq:min_max} using a fixed hyperparameter $\nu$
will converge to the $\boldsymbol{\delta}^{*}$such that $\overline{q}_{\boldsymbol{\delta}^{*}}(\boldsymbol{x})r(\boldsymbol{x};\boldsymbol{\delta}^{*})=\overline{p}_{\boldsymbol{\delta}^{*}}(\boldsymbol{x})$.
We also define the ``optimal'' weight assignment
$
w_{i}^{*} = 
\frac{1}{n_{p}}, \forall i,\boldsymbol{x}_{p}^{(i)}\in\overline{X}(\boldsymbol{\delta}^{*})
$
and 0 otherwise.
Interestingly, the constraint in \eqref{eq:min_max}, $\langle \boldsymbol{w}^{*}, \boldone \rangle = \nu $ may \emph{not} hold, but our analysis in this section suggests we can always find
a pair $(\hat{\bolddelta},\hat{\boldsymbol{w}})$ in the feasible
region so that $\|\hat{\boldsymbol{\delta}}-\boldsymbol{\delta}^{*}\|$
converges to 0 under mild conditions.

We first assume the log density ratio model 
and its CDF is Lipschitz continuous.
\begin{assumption}
	\label{ass.ratio.lip} 
	\begin{equation}
	\forall\boldsymbol{u}\in\mathrm{Ball}(\rho),\sup_{\boldsymbol{x}}\left|\hat{z}{}_{\boldsymbol{\delta}^{*}+\boldsymbol{u}}(\boldsymbol{x})-\hat{z}_{\boldsymbol{\delta}^{*}}(\boldsymbol{x})\right|\le C_{\mathrm{lip}}\|\boldsymbol{u}\|.\label{eq:lip}
	\end{equation}
	Define $T(\boldsymbol{u},\epsilon):=\left\{ \boldsymbol{x}\in\mathbb{\mathbb{R}^{\mathit{d}}}\mid\;|z_{\boldsymbol{\delta}^{*}}(\boldx)-t(\boldsymbol{\delta}^{*})|\le2C_{\mathrm{lip}}\|\boldsymbol{u}\|+\epsilon\right\} $
	where $0<\epsilon\le1$. We assume $\forall\boldsymbol{u}\in\mathrm{Ball}(\rho),0<\epsilon\le1$
	\[
	P\left[\boldsymbol{x}_{p}\in T(\boldsymbol{u},\epsilon)\right]\le C_{\mathrm{CDF}}\cdot\|\boldsymbol{u}\|+\epsilon.
	\]
\end{assumption}
In this assumption, we define a ``zone'' $T(\boldsymbol{u},\epsilon)$
near the $\nu$-th quantile $t(\boldsymbol{\delta}^{*})$ and assume
the CDF of our ratio model is upper-bounded over this region. Different from Assumption \ref{ass.ratio.lip-1}, the RHS of \eqref{eq:lip}
is with respect to $\ell_{2}$ norm of $\boldu$. In the following assumption, we assume regularity on $P$ and $Q$.
\begin{assumption}
	\label{assu:regP}
	$\forall \boldx_q \in \mathbb{R}^{d},  \|\boldsymbol{f}(\boldsymbol{x}_{q})\|_{\infty}\le C_{q}$   and $\forall\boldsymbol{u}\in\mathrm{Ball}(\rho),\forall\boldsymbol{x}_{p}\in T(\boldsymbol{u},1),\|\boldsymbol{f}(\boldsymbol{x}_{p})\|_{\infty}\le C_{p}$.
\end{assumption}
\begin{thm}
	\label{cor:truncationsetting}
	In addition Assumption \ref{assu:.boundedrq} and \ref{assu:RSC} and other mild assumptions (see Section \ref{sec:proofoftruncation} in the supplementary material), Assumption \ref{ass.ratio.lip} and
 \ref{assu:regP} hold. If $1\ge\nu\ge\frac{8C_{\mathrm{CDF}}\sqrt{k}C_{p}C_{r}^{2}}{\kappa'_{1}},n_q = \Omega(|\overline{X}{}^{p}(\boldsymbol{\delta}^{*})|^2)$,
\[\textstyle{
\lambda_{n}\ge2\max\left[\sqrt{\frac{K'_{1}\log d}{|\overline{X}{}^{p}(\boldsymbol{\delta}^{*})|}}+\frac{2C_{r}^{2}C_{q}|X_{q}\backslash\overline{X}^{q}(\boldsymbol{\delta}^{*})|}{n_{q}},\frac{2L\cdot C_{p}}{\sqrt{n_{p}}},\frac{\rho \nu c}{2C_{r}^{2}\sqrt{n_{q}}}\right],
}
\]
where $K'_{1}>0, c>0$ are constants,
we are guaranteed that $||\hat{\boldsymbol{\delta}}-\boldsymbol{\delta}^{*}\|\le\frac{4C_{r}^{2}}{\nu\kappa'_{1}}\cdot3\sqrt{k}\lambda_{n}$
with high probability.\label{thm:main-1-1} 
\end{thm} 
It can be seen that $\|\hat{\boldsymbol{\delta}}-\boldsymbol{\delta}^{*}\| = O\left(\sqrt{{\log d}/{\min(|\overline{X}^{p}(\boldsymbol{\delta}^{*})|, n_q)}}\right)$ if $d$ is reasonably large and $|X_{q}\backslash\overline{X}^{q}(\boldsymbol{\delta}^{*})|/{n_{q}}$ decays fast. 

\section{Experiments}
\label{sec.label}
\subsection{Detecting Sparse Structural Changes between Two Markov Networks (MNs) \cite{Liu2016a}}
In the first experiment\footnote{Code can be found at \url{http://allmodelsarewrong.org/software.html}}, we learn changes between two Gaussian MNs under the outlier setting. 
The ratio between two Gaussian MNs can be parametrized as $p(\boldx)/q(\boldx) \propto \exp(-\sum_{i,j \le d} \Delta_{i,j} x_i x_j)$, where $\Delta_{i,j} := \Theta^{p}_{i,j} - \Theta^{q}_{i,j}$ is the difference between precision matrices. 
We generate 500 samples as $X_p$ and $X_q$ using two randomly structured Gaussian MNs. 
One point $[10, \dots, 10]$ is added as an outlier to $X_p$. To induce sparsity, we set $R(\boldDelta) = \sum_{i,j=1, i\le j}^{d}{|\Delta_{i, j}|}$ and fix $\lambda = 0.0938$.  Then run DRE and TRimmed-DRE to learn the sparse \emph{differential} precision matrix $\boldDelta$ and results are plotted on Figure \ref{fig:mnchange2} and \ref{fig:mnchange1}\footnote{Figures are best viewed in color.} where the ground truth (the position $i,j, \Delta_{i,j}^* \neq 0$) is marked by red boxes. It can be seen that the outlier completely misleads DRE while TR-DRE performs reasonably well. We also run experiments with two different settings ($d=25, d=36$) and plot True Negative Rate (TNR) - True Positive Rate (TPR) curves. We fix $\nu$ in TR-DRE to 90\% and compare the performance of DRE and TR-DRE using DRE without any outliers as gold standard (see Figure \ref{fig:TNRTPR}). It can be seen that the added outlier makes the DRE fail completely while TR-DRE can almost reach the gold standard. It also shows the price we pay: TR-DRE does lose some power for discarding samples. However, the loss of performance is still acceptable. 
\begin{figure}
	\centering
	\subfloat[$\hat{\boldDelta}$ obtained by DRE, $d = 20$, with one outlier.]{
		\label{fig:mnchange2}\includegraphics[width = .33\textwidth]{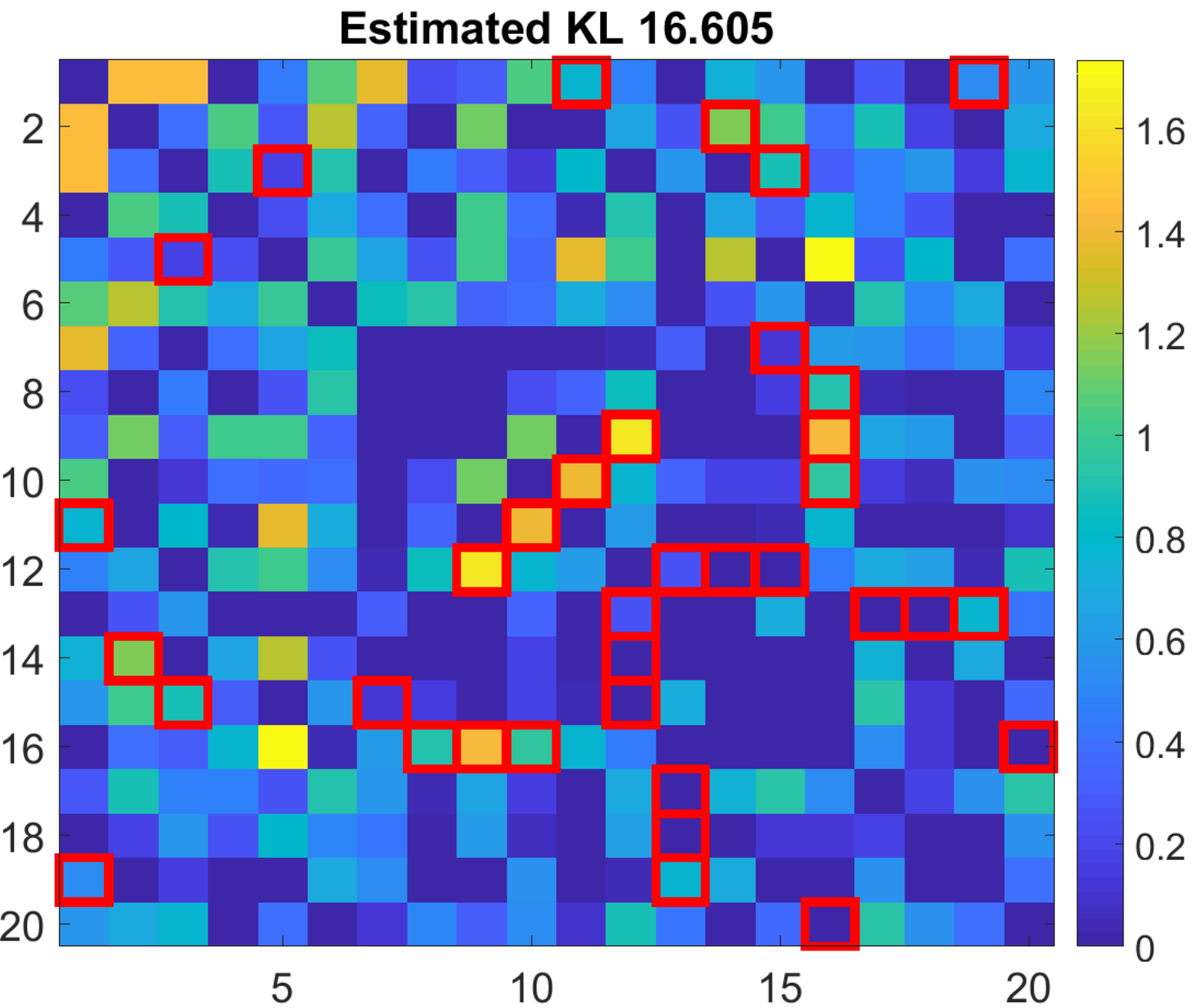}}
	\subfloat[$\hat{\boldDelta}$ obtained by TR-DRE, $\nu = 90\%$, with one outlier.]{\label{fig:mnchange1}\includegraphics[width = .33\textwidth]{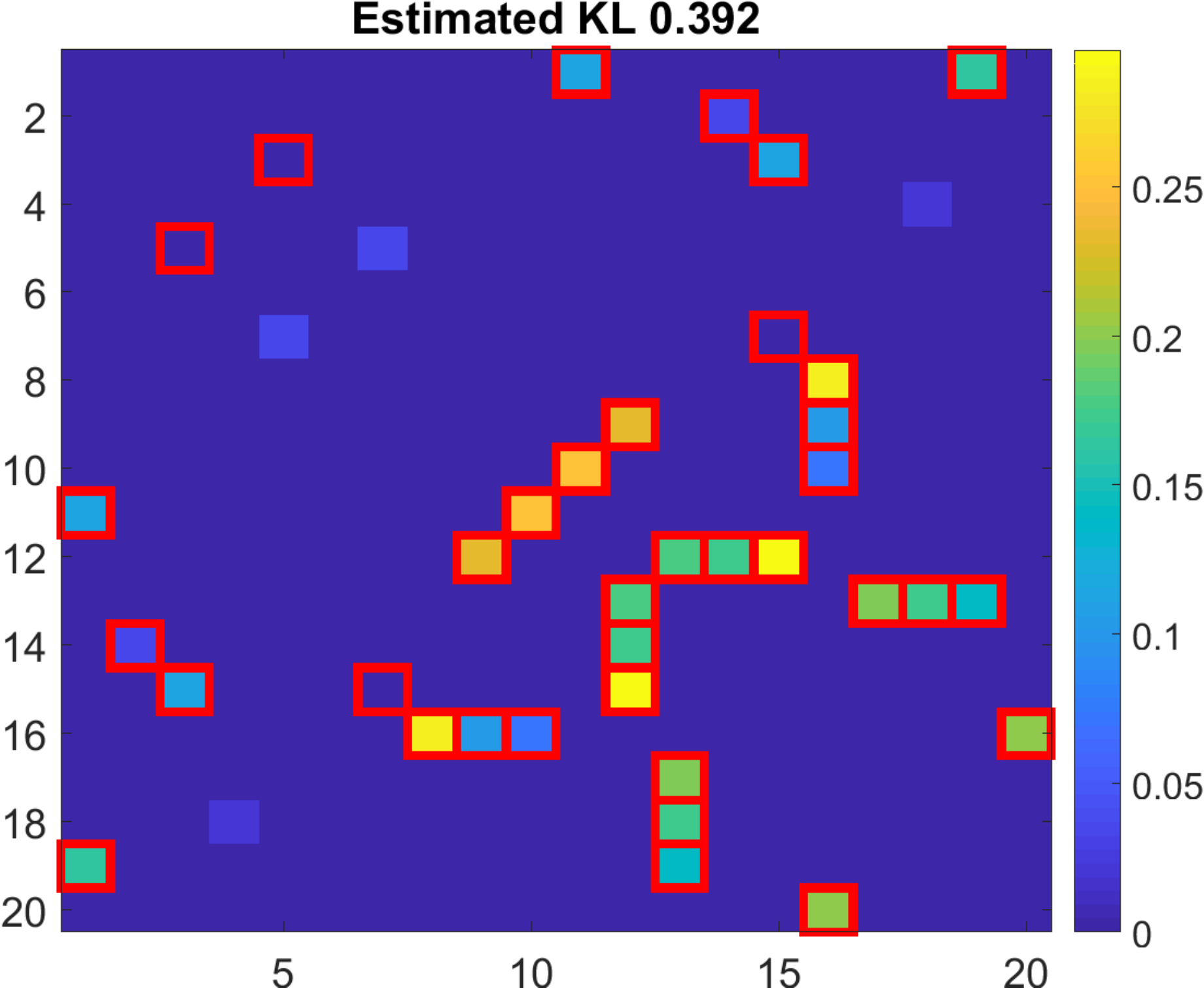}}
	\subfloat[TNR-TPR plot, $\nu=90\%$]{\label{fig:TNRTPR}\includegraphics[width = .33\textwidth]{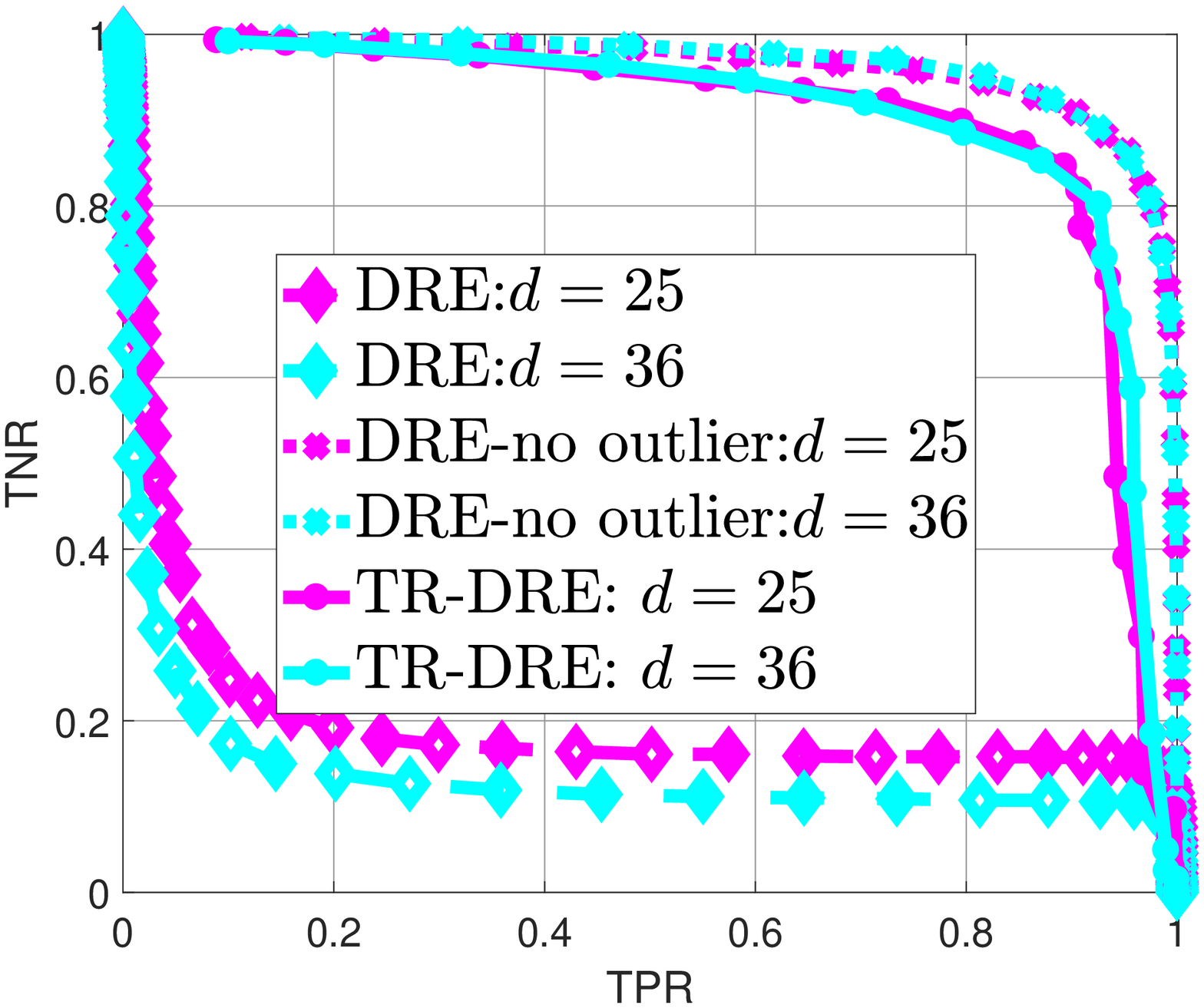}}
	\caption{Using DRE to learn changes between two MNs. We set $R(\cdot) = \|\cdot\|_1 $ and $ f(x_i, x_j) = x_i x_j$.	\label{fig.mnchange}}
\end{figure}
\subsection{Relative Novelty Detection from Images}
\begin{figure}
	\centering
	\subfloat[Dataset]{\label{fig:dataset}\includegraphics[width = .16\textwidth]{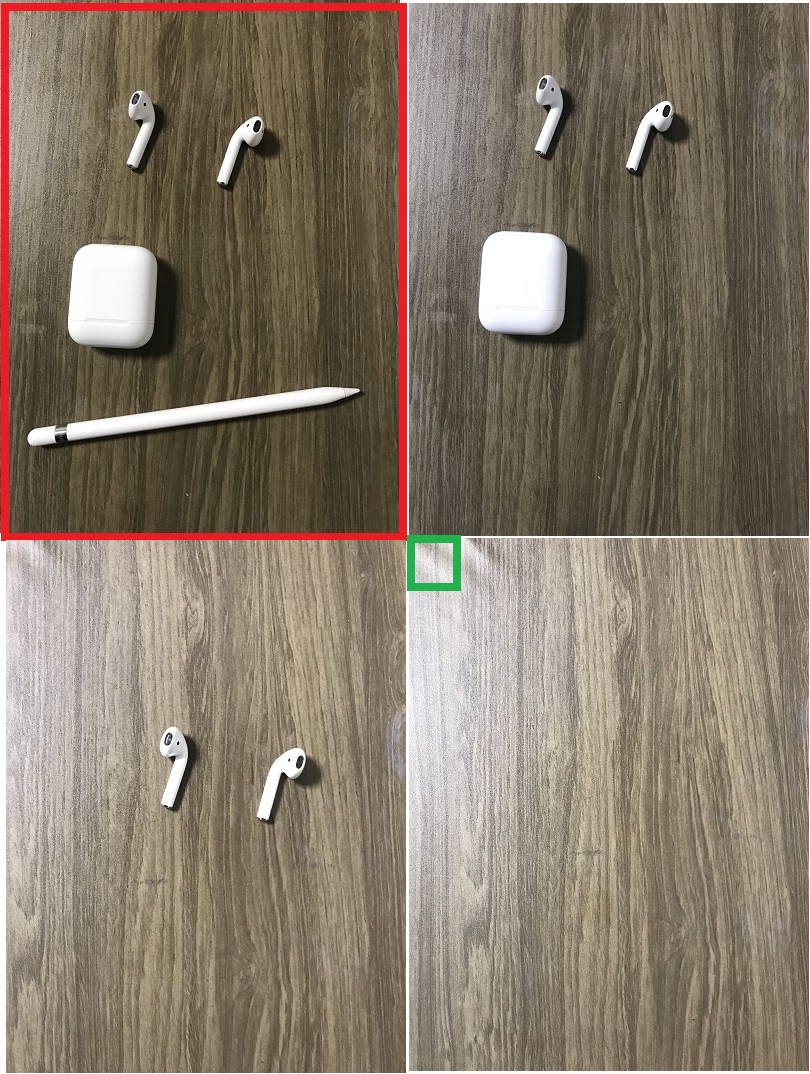}}
	\subfloat[$\nu = 97\%$]{\label{fig:detected97}\includegraphics[width = .16\textwidth]{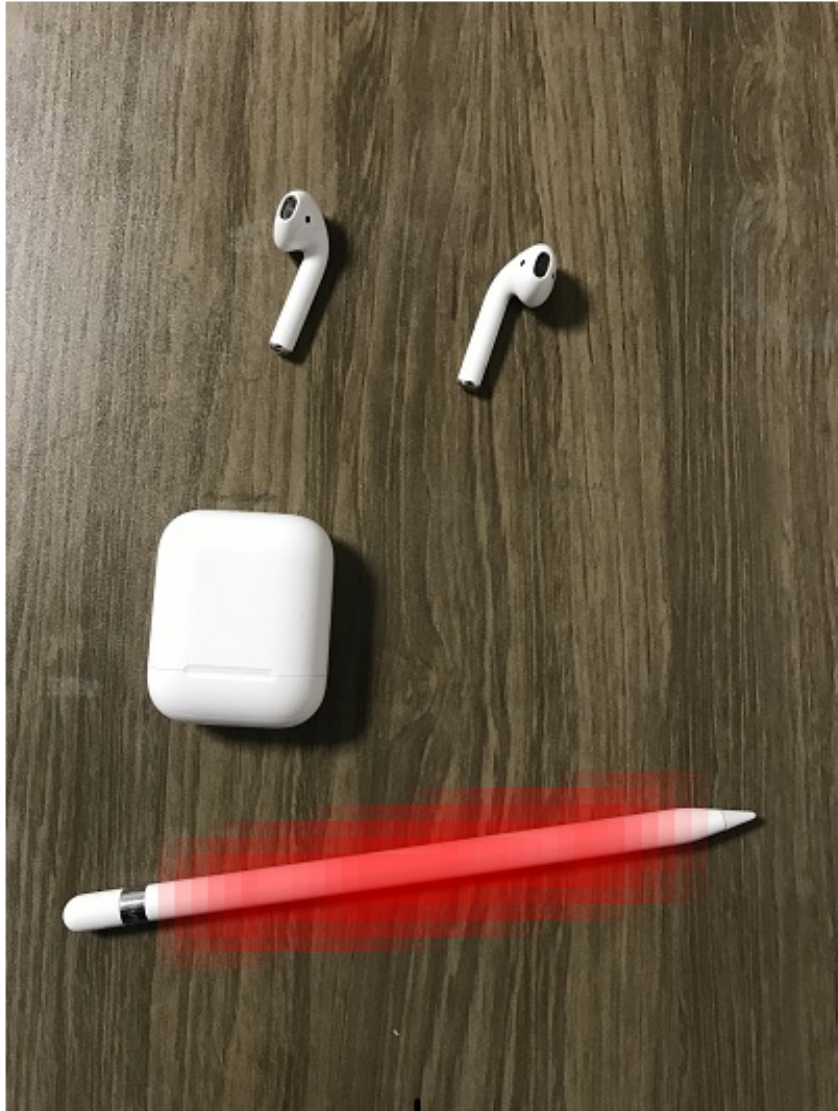}}
	\subfloat[$\nu = 90\%$]{\label{fig:detected90}\includegraphics[width = .16\textwidth]{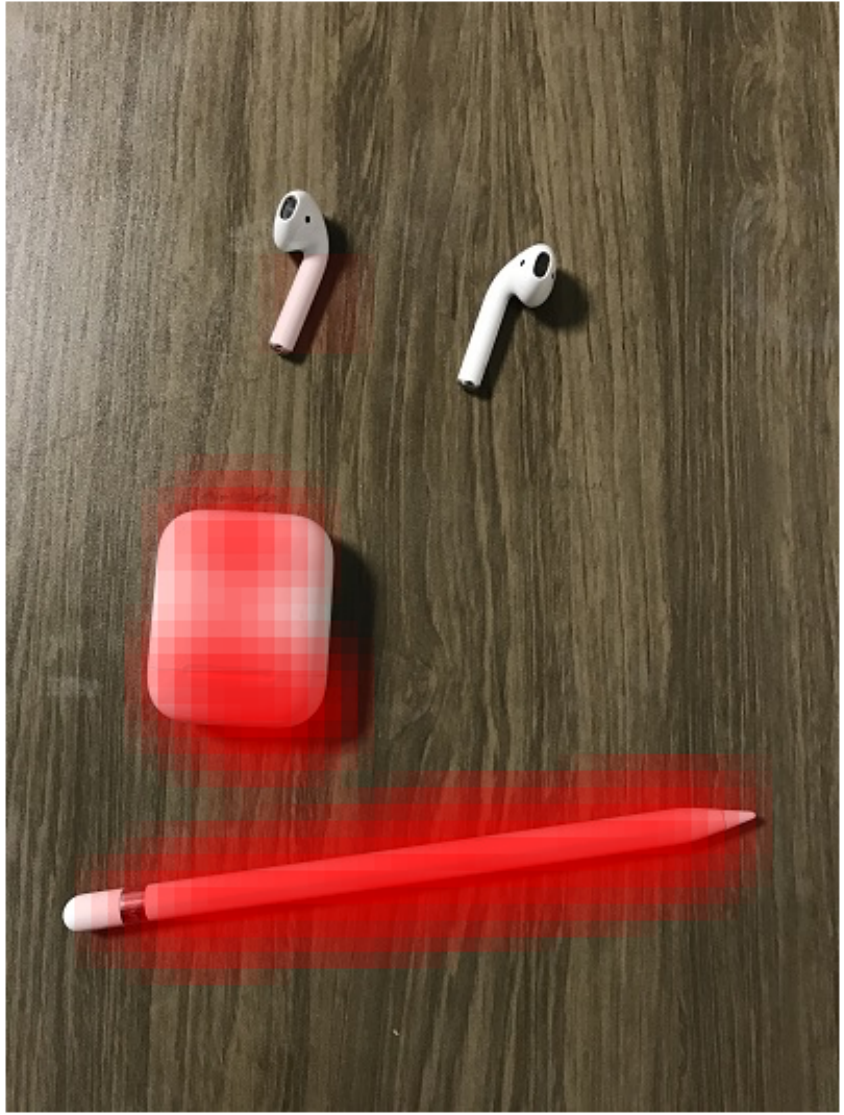}}
	\subfloat[$\nu = 85\%$]{\label{fig:detected85}\includegraphics[width = .16\textwidth]{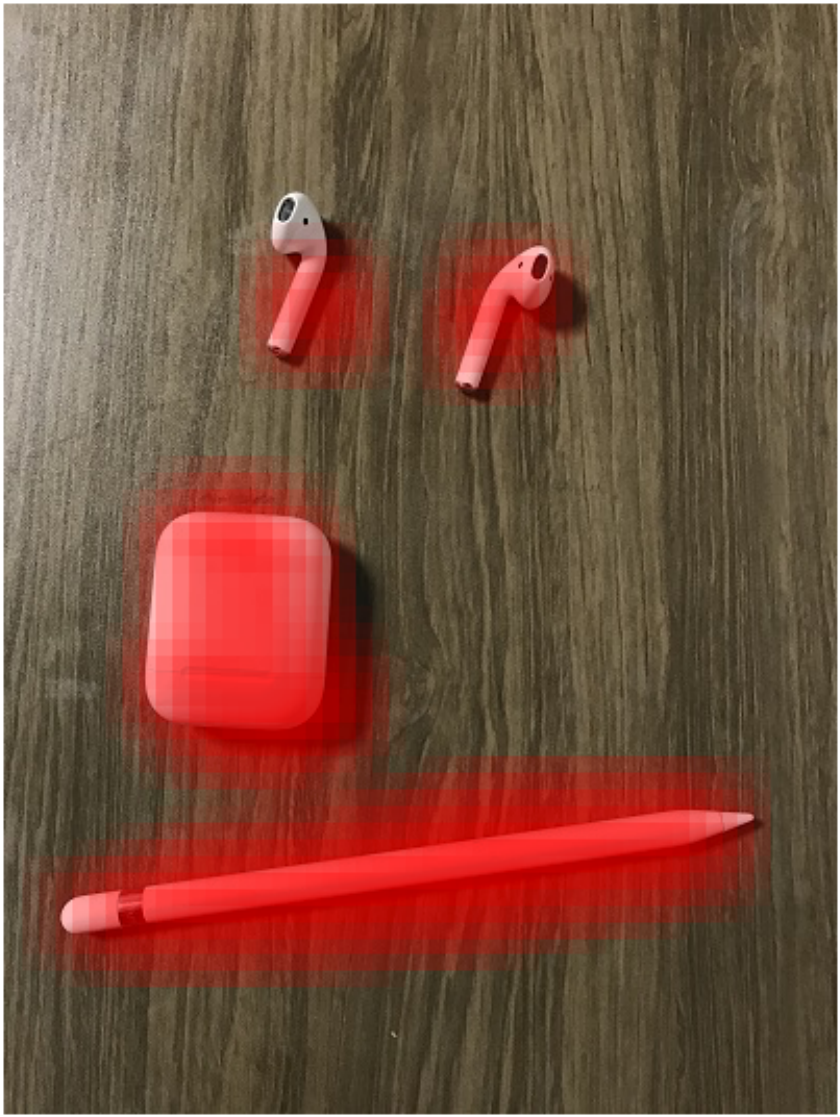}}
	\subfloat[TH-DRE]{\label{detectedsmola}\includegraphics[width = .16\textwidth]{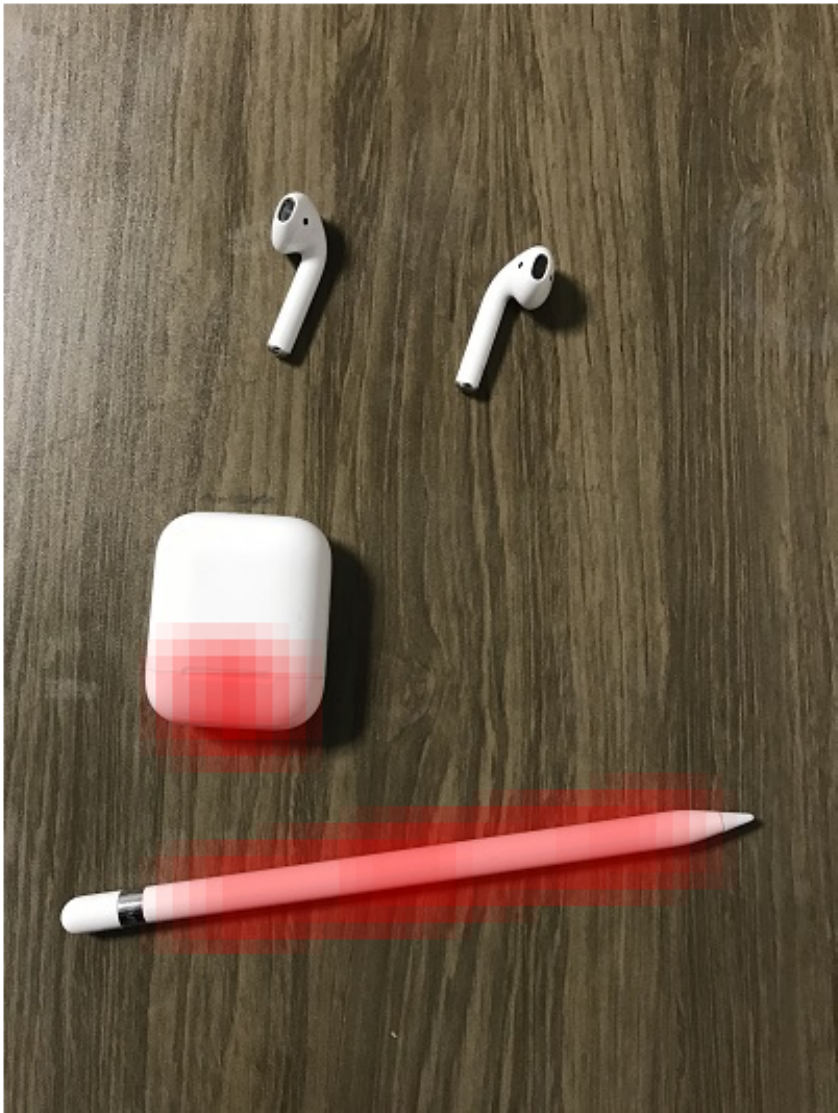}}
	\subfloat[one-SVM]{\label{detectedosvm}\includegraphics[width = .16\textwidth]{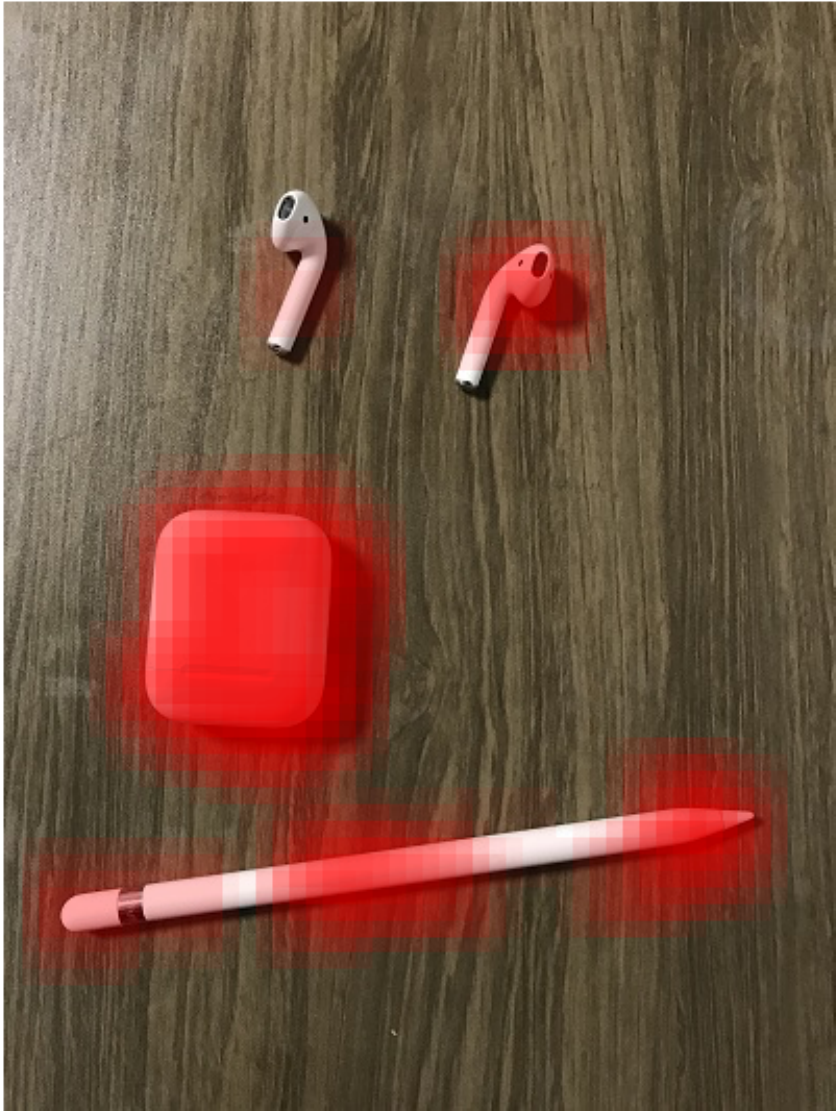}}
	\caption{Relative object detection using super pixels. We set $R(\cdot) = \|\cdot\|^2$, $ \boldf(\boldx)$ is an RBF kernel.}
\end{figure}
In the second experiment, 
we collect four images (see Figure \ref{fig:dataset}) containing three objects with a textured background: a pencil, an earphone and an earphone case. We create data points from these four images using sliding windows of $48 \times 48$ pixels (the green box on the lower right picture on Figure \ref{fig:dataset}). We extract 899 features using MATLAB HOG method on each window and construct an 899-dimensional sample.  Although our theorems in Section \ref{sec:-Consistency-in} are proved for linear models, here $\boldf(\boldx)$ is an RBF kernel using all samples in $X_p$ as kernel basis. We pick the top left image as $X_p$ and using all three other images as $X_q$,  then run TR-DRE, THresholded-DRE \cite{Smola2009}, and one-SVM.

In this task, we select high density ratio super pixels on image $X_p$. 
It can be expected that the super pixels containing the pencil will exhibit high density ratio values as they did not appear in the reference dataset $X_q$ while super pixels containing the earphone case, the earphones and the background, repeats similar patches in $X_q$ will have lower density ratio values. 
This is different from a conventional novelty detection, as a density ratio function help us capture only the relative novelty.
For TR-DRE, we use the trimming threshold $\hat{t}$ as the threshold for selecting high density ratio points. 

It can be seen on Figure \ref{fig:detected97}, \ref{fig:detected90} and \ref{fig:detected85}, as we tune $\nu$ to allow more and more high density ratio windows to be selected, more relative novelties are detected: First the pen, then the case, and finally the earphones, as the lack of appearance in the reference dataset $X_q$ elevates the density ratio value by different degrees.
In comparison, we run TH-DRE with top 3\% highest density ratio values thresholded, which corresponds to $\nu = 97\%$ in our method. The pattern of the thresholded windows (shaded in red) in Figure \ref{detectedsmola} is similar to Figure \ref{fig:detected97} though some parts of the case are mistakenly shaded. 
Finally, one-SVM with $3\%$ support vectors (see Figure \ref{detectedosvm}) does not utilize the knowledge of a reference dataset $X_q$ and labels all salient objects in $X_p$ as they corresponds to the ``outliers'' in $X_p$.  

\section{Conclusion}
\label{sec.concl}
We presents a robust density ratio estimator based on the idea of
trimmed MLE. 
It has a convex formulation
and the optimization can be easily conducted using a subgradient ascent method.
We also investigate its theoretical property through an equivalent
weighted M-estimator whose $\ell_{2}$ estimation error bound was
provable under two high-dimensional, robust settings. 
Experiments confirm the
effectiveness and robustness of the our trimmed estimator.
\section*{Acknowledgments}
We thank three anonymous reviewers for their detailed and helpful comments.  Akiko Takeda thanks Grant-in-Aid for Scientific Research (C), 15K00031. Taiji Suzuki was partially supported by MEXT KAKENHI (25730013, 25120012, 26280009 and 15H05707), JST-PRESTO and JST-CREST. Song Liu and Kenji Fukumizu have been supported in part by MEXT Grant-in-Aid for Scientific Research on Innovative Areas (25120012).

\bibliographystyle{plain}
\bibliography{main}

\newpage{}

\onecolumn

\section{Appendix}
To lighten the notation system, we drop the feature transform $\boldsymbol{f}$
from our equations. The analysis procedure does not change with or
without $\boldsymbol{f}$.

\subsection{Relationship between Trimmed DRE and Binary SVM \cite{Scholkopf2001,Cristianini2000}}
\label{sec:poofproposvm}
Consider a ``symmetrized'' extension to the criterion
\eqref{eq:MLEOBJ}: 
\begin{align}
\label{eq:symm1}
	\min_{\bolddelta}&\mathrm{KL}\left[p|q\cdot r_{\boldsymbol{\delta}}\right]+\mathrm{KL}\left[q|p\cdot1/r_{\boldsymbol{\delta}}\right] \notag \\
	\approx &c - \max_{\boldsymbol{\delta}}\frac{1}{n_{p}}\sum_{i=1}^{n_{p}}\log \hat{r}(\boldsymbol{x}_{p}^{(i)};\boldsymbol{\delta})+\frac{1}{n_{q}}\sum_{i=1}^{n_{q}}\log \hat{r}_{2}(\boldsymbol{x}_{q}^{(i)};\boldsymbol{\delta})
\end{align}
that jointly minimizes the KL divergence from $P$ to $Q$ and from
$Q$ to $P$. Similar to \eqref{eq:ratiomodel2}, we use $\hat{r}_{2}$ to model the ratio $q/p$: \[\hat{r}_{2}(\boldsymbol{x};\bolddelta)=\frac{\exp\langle-\bolddelta,\boldsymbol{x}\rangle}{\hat{N}_{2}(\bolddelta)},\hat{N}_{2}(\boldsymbol{\delta}):=\frac{1}{n_{p}}\sum_{j=1}^{n_{p}}\exp\langle\boldsymbol{-\delta},\boldsymbol{x}_{p}^{(j)}\rangle.\]
The minus in front of the $\bolddelta$ is due to the inversion of
the ratio. We can trim the objective function \eqref{eq:symm1} and add a regularization term $\lambda R(\bolddelta)$ as we did for the asymmetric one:
\begin{equation}
	\max_{\boldsymbol{\delta}}\frac{1}{n_{p}}\sum_{i=1}^{n_{p}}[\log \hat{r}(\boldsymbol{x}_{p}^{(i)};\boldsymbol{\delta}) - t_0]_- + \frac{1}{n_{q}}\sum_{i=1}^{n_{q}}[\log \hat{r}_{2}(\boldsymbol{x}_{q}^{(i)};\boldsymbol{\delta} )- t_0]_- - \lambda R(\bolddelta) 
\label{eq:symm}
\end{equation}
\begin{prop}
	\label{prop.svm}
	If $n_p=n_q, t_0 = 1, R(\cdot) = \|\cdot\|^2_2$, the maximizer $\hat{\bolddelta}$ of \eqref{eq:symm} is the same as the primal solution of a modified SVM using $X_p$ and $X_q$ as positive and negative class respectively.
\end{prop}
It suggests
SVM learns an \emph{unnormalized }and \emph{trimmed} density
ratio function as the decision function. 

\begin{proof}
By introducing
the slack variables as we did in \eqref{eq:objsimple}.
\eqref{eq:symm} can be rewritten as:
\begin{align}
\label{eq:symmsvm}
\min_{\boldsymbol{\delta},\boldsymbol{\epsilon}\ge\boldsymbol{0}}\notag & \frac{1}{n_{p}}\langle\epsilon_{p},\boldsymbol{1}\rangle+\frac{1}{n_{q}}\langle\epsilon_{q},\boldsymbol{1}\rangle+\lambda R(\boldsymbol{\delta})\\
\text{s.t.}\; \notag & \forall\boldsymbol{x}_{p}^{(i)}\text{\ensuremath{\in}}X_{p},\forall\boldsymbol{x}_{q}^{(i)}\text{\ensuremath{\in}}X_{q},\\
&\log\hat{r}(\boldsymbol{x}_{p}^{(i)};\boldsymbol{\delta})\ge t_0-\epsilon_{p,i},\notag \\
&\log\hat{r}_{2}(\boldsymbol{x}_{q}^{(i)};\boldsymbol{\delta})\ge t_0-\epsilon_{q,i},
\end{align}
After substituting $\hat{r}$ and $\hat{r}_{2}$, \eqref{eq:symmsvm} can
be re–written as
\begin{align}
\label{eq.augsvm}
\min_{\boldsymbol{\delta},\boldsymbol{\epsilon}\ge\boldsymbol{0}}\notag & \frac{1}{n_{p}}\langle\epsilon_{p},\boldsymbol{1}\rangle+\frac{1}{n_{q}}\langle\epsilon_{q},\boldsymbol{1}\rangle+\lambda R(\boldsymbol{\delta})\notag\\
\text{s.t.}\; & \forall\boldsymbol{x}_{p}^{(i)}\text{\ensuremath{\in}}X_{p},\forall\boldsymbol{x}_{q}^{(i)}\text{\ensuremath{\in}}X_{q},\notag\\
& \langle\boldsymbol{\delta},\boldx_p^{(i)}\rangle-\log\hat{N}(\boldsymbol{\delta})\ge t_0-\epsilon_{p,i},\notag\\
&\langle-\boldsymbol{\delta},\boldx_q^{(i)}\rangle-\log\hat{N}_{2}(\boldsymbol{\delta})\ge t_0-\epsilon_{q,i}.
\end{align}
Let $n_{p}=n_{q}, t_0 = 1$, $R(\bolddelta) = \|\bolddelta\|^2$, 
\eqref{eq.augsvm} is an SVM (without a bias term) using $X_{p}$ and $X_{q}$
as positive and negative samples respectively, except the presences of log normalization terms $\log\hat{N}(\boldsymbol{\delta})$ and $\log\hat{N}_{2}(\boldsymbol{\delta})$.
\end{proof}

\subsection{Proof of Proposition\label{subsec:minma} \ref{prop:minma}}
\begin{proof}
To prove the statement, we construct the dual of \eqref{eq:objquantile}
which has the exactly same form as \eqref{eq:min_max}. Denote $\boldsymbol{X}_{p}=\left[\boldsymbol{x}_{p}^{(1)},\dots,\boldsymbol{x}_{p}^{(n_{p})}\right]\in\mathbb{R}^{d\times n_{p}}$
and $\boldsymbol{X}_{q}=\left[\boldsymbol{x}_{q}^{(1)},\dots,\boldsymbol{x}_{q}^{(n_{q})}\right]\in\mathbb{R}^{d\times n_{q}}$.

The Lagrangian of \eqref{eq:objquantile} can be written as 
\begin{align}
l(\boldsymbol{\alpha},\boldsymbol{\alpha}',\alpha'',\boldsymbol{\delta},t\boldsymbol{,\epsilon})=-\langle\boldsymbol{\alpha},\boldsymbol{\delta^{\top}X}_{p}-\log\widehat{N}(\boldsymbol{\delta}) \cdot \boldone-t\cdot\boldsymbol{1}+\boldsymbol{\epsilon}\rangle\nonumber \\
-\langle\boldsymbol{\alpha}',\boldsymbol{\epsilon}\rangle-\alpha''\cdot t+\frac{1}{n_{p}}\langle\boldsymbol{\epsilon},\boldsymbol{1}\rangle-\nu\cdot t+\lambda R(\boldsymbol{\delta})\label{eq:lagrangian}
\end{align}
where $\boldsymbol{\alpha}\in\mathbb{R}_{+}^{n_{p}},\boldsymbol{\alpha}'\in\mathbb{R}_{+},\alpha''\in\mathbb{R}_{+}.$
Now we analyze the KKT condition of the above Lagrangian.

Suppose the optimal $\hat{t}>0$\footnote{if $t=0$ is the optimal and assume $R(\boldzero) = 0$, we only have a trivial solution $\boldsymbol{\delta}=\boldsymbol{0},\boldsymbol{\epsilon}=\boldsymbol{0}$,
which is easy to verify and rules out.}, then $\alpha''=0$ by the slackness condition that $t'\alpha'' = 0$. The optimality condition of $t$ in \eqref{eq:lagrangian}
yields: 
\begin{equation}
\nabla_{t}l(\boldsymbol{\alpha},\boldsymbol{\alpha}',\alpha'',\boldsymbol{\delta,}t\boldsymbol{,\epsilon})=\left\langle \boldsymbol{\alpha},\boldsymbol{1}\right\rangle -\nu=0\rightarrow\sum_{i=1}^{n_{p}}\alpha_{i}=\nu,\label{eq:optimality_t}
\end{equation}
and the optimality condition of $\boldsymbol{\epsilon}$ yields 
\begin{align}
\nabla_{\boldsymbol{\epsilon}}l(\boldsymbol{\alpha},\boldsymbol{\alpha}',\alpha'',\boldsymbol{\delta},t\boldsymbol{,}\boldsymbol{\epsilon})=\boldsymbol{0} & \rightarrow-\boldsymbol{\alpha}-\boldsymbol{\alpha}'+\frac{1}{n_{p}}\cdot\mathbf{1}=\boldsymbol{0}\label{eq:optimal.alpha}
\end{align}
From \eqref{eq:optimality_t} and \eqref{eq:optimal.alpha}, and the
slackness condition of optimization \eqref{eq:objsimple}, we can
see $\boldsymbol{x}_{p}^{(i)}\in\mathbf{X}_{p}$, if $\log\hat{r}(\boldsymbol{x}_{p}^{(i)};\boldsymbol{\delta})<t$,
then $\epsilon_{i}>0$ which leads to $\alpha'_{i}=0$ (the constraint
of $\epsilon_{i}\ge0$ is ineffective) and thus $\alpha_{i}=\frac{1}{n_{p}}$.

In contrast, if $\log\hat{r}(\boldsymbol{x}_{p}^{(i)};\boldsymbol{\delta})>t,$
then we have $\alpha_{i}=0,\epsilon_{i}=0$ (the constraint of $\epsilon_{i}\ge0$
is \emph{effective}). If $\log\hat{r}(\boldsymbol{x}_{p}^{(i)};\boldsymbol{\delta})$
falls right on the boundary $t$, i.e., $\log r(\boldsymbol{x}_{p}^{(i)};\boldsymbol{\delta})=t$,
$\alpha_{i}\in[0,\frac{1}{n_{p}}]$, since the KKT condition $\epsilon_{i}\alpha'_{i}=0$
indicating $\alpha'_{i}$ can take non-negative values as long as
$\frac{1}{n_{p}}\cdot\mathbf{1}=\boldsymbol{\boldsymbol{\alpha}+\boldsymbol{\alpha}'}$.
We summarize: 
\begin{equation}
\begin{cases}
\alpha_{i}=\frac{1}{n_{p}} & \log\hat{r}(\boldsymbol{x}_{p}^{(i)};\boldsymbol{\delta})<t\\
0\le\alpha_{i}\le\frac{1}{n_{p}} & \log\hat{r}(\boldsymbol{x}_{p}^{(i)};\boldsymbol{\delta})=t\\
\alpha_{i}=0 & \log\hat{r}(\boldsymbol{x}_{p}^{(i)};\boldsymbol{\delta})>t.
\end{cases}\label{eq:alpha_conditions}
\end{equation}

It can be observed that for \eqref{eq:objquantile}, $(\boldsymbol{\delta}=\boldsymbol{0},\boldsymbol{\epsilon}=0.2\cdot\boldsymbol{1},t=0.1)$
is a feasible interior point, and it makes all inequality constraints
strict, so the \emph{Slater's condition} holds for our original primal
problem which is also convex. Therefore, the lagrangian dual of the
original problem \eqref{eq:objquantile} is 
\begin{align}
 & \min_{\boldsymbol{\delta}}\max_{\boldsymbol{\alpha}\ge0,\boldsymbol{\alpha}'\ge0,\alpha''\ge0}\min_{\boldsymbol{\epsilon},t}\;l(\boldsymbol{\alpha},\boldsymbol{\alpha}',\alpha'',\boldsymbol{\delta,}t\boldsymbol{,\epsilon})\nonumber \\
= & \min_{\boldsymbol{\delta}}\max_{\boldsymbol{\alpha}}-\langle\boldsymbol{\alpha},\boldsymbol{\delta^{\top}X_{p}}-\log\widehat{N}(\boldsymbol{\delta})\rangle+\lambda R(\boldsymbol{\delta})\\
s.t. & \boldsymbol{\alpha}\in\left[0,\frac{1}{n_{p}}\right]^{n_{p}},\langle\boldsymbol{1},\boldsymbol{\alpha}\rangle=\nu.\label{eq:lagrangiandual}
\end{align}
which is the same as \eqref{eq:min_max} and any points satisfy the
KKT condition are both dual \eqref{eq:lagrangiandual} and primal
\eqref{eq:objquantile} optimal. 
\end{proof}

\subsection{Lemma \ref{lem:RSC}}
\begin{lem}
	\label{lem:RSC}If Assumptions \ref{assu:.boundedrq} and \ref{assu:RSC}
	hold, then 
	\begin{align}
	\label{eq:RSClem}
	-\boldsymbol{u}^{\top}\nabla_{\boldsymbol{\delta}}^{2}\mathcal{L}(\boldsymbol{\delta}^{*}+\boldsymbol{u},\boldsymbol{w}^{*})\boldsymbol{u}\ge\frac{\nu\kappa'_{1}}{2C_{r}^{2}}\|\boldsymbol{u}\|^{2}-\frac{\nu c}{2C_{r}^{2}}\cdot\frac{\|\boldsymbol{u}\|_{1}^{2}}{\sqrt{n_{q}}},
	\end{align}
\end{lem}
where $c$ is the constant determined by Assumption
\ref{assu:RSC}. 

\begin{proof}
	First, we write down -$\nabla_{\boldsymbol{\delta}}^{2}\mathcal{L}(\boldsymbol{\delta}^{*}+\boldsymbol{u},\boldsymbol{w}^{*})$:
	\begin{align*}
	-\nabla_{\boldsymbol{\delta}}^{2}\mathcal{L}(\boldsymbol{\delta}^{*}+\boldsymbol{u},\boldsymbol{w}^{*}) & =-\nabla^{2}\sum_{i=1}^{n_{p}}w_{i}^{*}\cdot\log\hat{r}(\boldsymbol{x}_{p}^{(i)};\boldsymbol{\delta^{*}+u})\\
	&=-\sum_{i=1}^{n_{p}}w_{i}^{*}\cdot\nabla ^2  \log \widehat{N}(\bolddelta^*+\boldu)\\
	& = -\nu\cdot\nabla^{2}\log \widehat{N}(\bolddelta^{*}+\boldu),\\
	& = \nu\cdot\sum_{i=1}^{n_{q}}\frac{e^{(i)}}{s}\cdot\boldsymbol{x}_{q}^{(i)}\left(\boldsymbol{x}_{q}^{(i)}\right)^{\top}-\nu\cdot\left\{ \sum_{i=1}^{n_{q}}\frac{e^{(i)}}{s}\cdot\left(\boldsymbol{x}_{q}^{(i)}\right)\right\} \left\{ \sum_{i=1}^{n_{q}}\frac{e^{(i)}}{s}\cdot\left(\boldsymbol{x}_{q}^{(i)}\right)\right\} ^{\top}
	\end{align*}
	where $e^{(j)}:=\exp\left[\langle\boldsymbol{\delta}^{*}+\boldsymbol{u},\boldsymbol{x}^{(j)}\rangle\right],s:=\sum_{j=1}^{n_{q}}e^{(j)}$.
	\begin{align*}
	& \nu\boldsymbol{u}^{\top}\left\{ \sum_{i=1}^{n_{q}}\frac{e^{(i)}}{s}\cdot\boldsymbol{x}^{(i)}\left(\boldsymbol{x}^{(i)}\right)^{\top}-\left\{ \sum_{i=1}^{n_{q}}\frac{e^{(i)}}{s}\cdot\left(\boldsymbol{x}^{(i)}\right)\right\} \left\{ \sum_{i=1}^{n_{q}}\frac{e^{(i)}}{s}\cdot\left(\boldsymbol{x}^{(i)}\right)\right\} ^{\top}\right\} \boldsymbol{u}\\
	= & \frac{\nu}{2}\boldsymbol{u}^{\top}\left\{ \sum_{i=1}^{n_{q}}\sum_{j\ne i}\frac{e^{(i)}e^{(j)}}{s^{2}}\left(\boldsymbol{x}^{(i)}-\boldsymbol{x}^{(j)}\right)\left(\boldsymbol{x}^{(i)}-\boldsymbol{x}^{(j)}\right)^{\top}\right\} \boldsymbol{u}
	\end{align*}
	Due to Assumption \ref{assu:.boundedrq}, $\frac{e^{(i)}}{s}\ge\frac{1}{C_{r}n_{q}}$. Let  $\xi_{i,j}=\left(\boldsymbol{x}^{(i)}-\boldsymbol{x}^{(j)}\right)\left(\boldsymbol{x}^{(i)}-\boldsymbol{x}^{(j)}\right)^{\top},$
	then we have the following inequalities 
	\begin{align*}
	\frac{\nu}{2}\boldsymbol{u}^{\top}\left\{ \sum_{i=1}^{n_{q}}\sum_{j\ne i}\frac{e^{(i)}e^{(j)}}{s^{2}}\xi_{i,j}\right\} \boldsymbol{u}
	\ge \frac{\nu}{2C_{r}^{2}}\boldsymbol{u}^{\top}\left\{ \frac{1}{n_{q}^{2}}\sum_{i=1}^{n_{q}}\sum_{j\ne i}\xi_{i,j}\right\} \boldsymbol{u}
	= \frac{\nu}{2C_{r}^{2}}\boldsymbol{u}^{\top} \mathrm{cov}(\boldX_q) \boldsymbol{u}
	\end{align*}
%
	We then invoke Assumption \ref{assu:RSC} to obtain
	$
	\frac{\nu}{2C_{r}^{2}}\boldsymbol{u}^{\top} \mathrm{cov}(\boldX_q) \boldsymbol{u} \ge \frac{\nu\kappa'_{1}}{2C_{r}^{2}}\|\boldsymbol{u}\|^{2}-\frac{\nu c}{2C_{r}^{2}\sqrt{n_{q}}}\|\boldsymbol{u}\|_{1}^{2}.
	$
\end{proof}

\subsection{Proof of Theorem\label{subsec:theorem1} \ref{thm:main}}
\begin{proof}
First, we define the $S$ and $S^{c}$ are the set of indices of non-zero
and zero elements of $\boldsymbol{\bolddelta}^{*}$. The cardinlity
of $S$ is $k$.

Define $\hat{\boldsymbol{u}}:=\hat{\boldsymbol{\delta}}-\boldsymbol{\delta}^{*}$.
From the Lemma \ref{lem:RSC} we can see that, 
\begin{align*}
	\langle\nabla_{\boldsymbol{\delta}}\mathcal{L}(\hat{\boldsymbol{\delta}},\boldsymbol{w}^{*})-\nabla_{\boldsymbol{\delta}}\mathcal{L}(\boldsymbol{\delta}^{*},\boldsymbol{w}^{*}),\hat{\boldsymbol{u}}\rangle \ge\kappa_{1}\|\hat{\boldsymbol{u}}\|^{2}-\tau_{1}(n,d)\|\hat{\boldsymbol{u}}\|_{1}^{2},
\end{align*}
where we set $\kappa_{1}:=\frac{\nu\kappa'_{1}}{2C_{r}^{2}},\tau_{1}(n,d):=\frac{\nu c}{2C_{r}^{2}\sqrt{n_{q}}}$. Using Holder's inequality,
\begin{align*}
\langle\nabla_{\boldsymbol{\delta}}\mathcal{L}(\hat{\boldsymbol{\delta}},\boldsymbol{w}^{*}),\hat{\boldsymbol{u}}\rangle+\|\nabla_{\boldsymbol{\delta}}\mathcal{L}(\boldsymbol{\delta}^{*},\boldsymbol{w}^{*})\|_{\infty}\|\hat{\boldsymbol{u}}\|_{1} & +\tau_{1}(n,d)\rho\|\hat{\boldsymbol{u}}\|_{1}\ge\kappa_{1}\|\hat{\boldsymbol{u}}\|^{2}.
\end{align*}
The introduction of $\rho$ is due to the bounded optimization region.
Due to \eqref{eq:ass2}, we can convert the above inequality
into 
\begin{align*}
\langle\nabla_{\boldsymbol{\delta}}\mathcal{L}(\hat{\boldsymbol{\boldsymbol{\delta}}},\hat{\boldsymbol{w}}),\hat{\boldsymbol{u}}\rangle+\kappa_{2}\|\hat{\boldsymbol{u}}\|^{2}+\tau_{2}(n,d)\|\hat{\boldsymbol{u}}\|_{1}+\|\nabla_{\boldsymbol{\delta}}\mathcal{L}(\boldsymbol{\delta}^{*},\boldsymbol{w}^{*})\|_{\infty}\|\hat{\boldsymbol{u}}\|_{1}+\rho\tau_{1}(n,d)\|\hat{\boldsymbol{u}}\|_{1}\ge\kappa_{1}\|\hat{\boldsymbol{u}}\|^{2},
\end{align*}
and because of the setting of $\lambda_{n}$, 
\begin{align}
\langle\nabla_{\boldsymbol{\delta}}\mathcal{L}(\hat{\boldsymbol{\delta}},\hat{\boldsymbol{w}}),\hat{\boldsymbol{u}}\rangle+\frac{\lambda_{n}}{2}\|\hat{\boldsymbol{u}}\|_{1}\ge(\kappa_{1}-\kappa_{2})\|\hat{\boldsymbol{u}}\|^{2},\label{eq:ub1}
\end{align}

Note that in the first term, $\hat{\boldsymbol{\delta}}$ is obtained
at the stationary condition, which implies that there is a subgradient, denoted by $\nabla\|\hat{\bolddelta}\|_1$, such that
\begin{align*}
\nabla_{\boldsymbol{\delta}}\mathcal{L}(\hat{\boldsymbol{\delta}},\hat{\boldsymbol{w}})=-\lambda_{n}\nabla_{\boldsymbol{\delta}}\|\hat{\boldsymbol{\delta}}\|_{1}=-\lambda_{n}\nabla_{\boldsymbol{\delta}}\|\hat{\boldsymbol{u}}+\boldsymbol{\delta}^{*}\|_{1,}
\end{align*}
(the second $\nabla$ is the subgradient notation) thus we can obtain
the upper-bound of $\langle\nabla_{\boldsymbol{\delta}}\mathcal{L}(\hat{\boldsymbol{\delta}},\hat{\boldsymbol{w}}),\boldsymbol{\hat{u}}\rangle$
using the following standard procedure: 
\begin{align}
\langle\nabla_{\boldsymbol{\delta}}\mathcal{L}(\hat{\boldsymbol{\delta}},\hat{\boldsymbol{w}}),\boldsymbol{\hat{u}}\rangle & =-\lambda_{n}\langle\nabla_{\boldsymbol{\delta}}\|\hat{\boldsymbol{u}}+\boldsymbol{\delta}^{*}\|_{1},\hat{\boldsymbol{u}}\rangle\nonumber \\
 & \le-\lambda_{n}(\|\hat{\boldsymbol{\delta}}\|_{1}-\|\boldsymbol{\delta}^{*}\|_{1})\text{ due to convexity of }\|\bolddelta\|_1 \text{ and the definition of subgradient. }\nonumber \\
 & =\lambda_{n}(\|\boldsymbol{\delta}^{*}\|_{1}+\|\hat{\boldsymbol{u}}_{S^{c}}\|_{1}-\|\hat{\boldsymbol{u}}_{S^{c}}\|_{1}-\|\hat{\boldsymbol{\delta}}\|_{1})\nonumber \\
 & =\lambda_{n}(\|\boldsymbol{\delta}^{*}+\hat{\boldsymbol{u}}_{S^{c}}\|_{1}-\|\hat{\boldsymbol{u}}_{S^{c}}\|_{1}-\|\hat{\boldsymbol{\delta}}\|_{1})\nonumber \\
 & =\lambda_{n}(\|\boldsymbol{\delta}^{*}+\hat{\boldsymbol{u}}_{S^{c}}\|_{1}+\|\hat{\boldsymbol{u}}_{S}\|_{1}-\|\hat{\boldsymbol{u}}_{S}\|_{1}-\|\hat{\boldsymbol{u}}_{S^{c}}\|_{1}-\|\hat{\boldsymbol{\delta}}\|_{1})\nonumber \\
 & \le\lambda_{n}(\|\boldsymbol{\delta}^{*}+\hat{\boldsymbol{u}}_{S}+\hat{\boldsymbol{u}}_{S^{c}}\|_{1}+\|\hat{\boldsymbol{u}}_{S}\|_{1}-\|\hat{\boldsymbol{u}}_{S^{c}}\|_{1}-\|\hat{\boldsymbol{\delta}}\|_{1})\nonumber \\
 & \le\lambda_{n}(\|\hat{\boldsymbol{u}}_{S}\|_{1}-\|\hat{\boldsymbol{u}}_{S^{c}}\|_{1})\label{eq:ub2}
\end{align}
Combining \eqref{eq:ub1} and \eqref{eq:ub2} we have 
\begin{align}
\lambda_{n}(\|\hat{\boldsymbol{u}}_{S}\|_{1}-\|\hat{\boldsymbol{u}}_{S^{c}}\|_{1})+\frac{\lambda_{n}}{2}\|\boldsymbol{\hat{u}}\|_{1} & \ge(\kappa_{1}-\kappa_{2})\|\boldsymbol{\hat{u}}\|^{2}\nonumber \\
\frac{3\lambda_{n}}{2}\|\hat{\boldsymbol{u}}_{S}\|_{1}-\frac{\lambda_{n}}{2}\|\hat{\boldsymbol{u}}_{S^{c}}\|_{1} & \ge(\kappa_{1}-\kappa_{2})\|\boldsymbol{\hat{u}}\|^{2}\label{eq:cs3}\\
\frac{3\lambda_{n}\sqrt{k}}{2}\|\hat{\boldsymbol{u}}\|_{2} & \ge(\kappa_{1}-\kappa_{2})\|\boldsymbol{\hat{u}}\|^{2}\nonumber \\
\frac{1}{(\kappa_{1}-\kappa_{2})}\cdot\frac{3\sqrt{k}\lambda_{n}}{2} & \ge\|\boldsymbol{\hat{u}}\|.\nonumber 
\end{align}
Substituting $\kappa_1$ and $\tau_{1}(n,d)$ according to Lemma \ref{lem:RSC}, we have the conclusion in Theorem \ref{thm:main}.
\end{proof}

\subsection{Proof of Theorem \ref{cor:outlier}}
\label{sec:proofcol1}
Now let's specify $\kappa_{2}$ and $\tau_{2}$ in Theorem
\ref{thm:main} under the outlier setting and derive the consistency.


Let's consider \eqref{eq:ass2}. It is easy to see that 
\[
\nabla_{\boldsymbol{\delta}}\mathcal{L}(\hat{\bolddelta},\hat{\boldw})-\nabla_{\boldsymbol{\delta}}\mathcal{L}(\hat{\bolddelta},\boldsymbol{w}^{*})=\sum_{i\in\hat{G}}w_{i}\boldsymbol{f}(\boldx_{p}^{(i)})-\frac{1}{n_{p}}\sum_{i\in G}\boldsymbol{f}(\boldx_{p}^{(i)}),\text{where }\hat{G}:=\{\boldx_{p}^{(i)}|\hat{w}_{i}\neq0\}.
\]
It is obvious that if $\hat{G}\equiv G$ and $\forall {i\in\hat{G}}, \hat{w}_{i} = \frac{1}{np}$, and $\forall i\in B$, $\hat{w}_i = 0$,
$\nabla_{\boldsymbol{\delta}}\mathcal{L}(\hat{\bolddelta},\hat{\boldw})-\nabla_{\boldsymbol{\delta}}\mathcal{L}(\hat{\bolddelta},\boldsymbol{w}^{*})=0$.
 
\begin{lem}
	If there exists a ``clearance'' between the good samples
	and the bad samples, such that $\min_{j\in B}z_{\boldsymbol{\delta}^{*}}(\boldsymbol{x}_{p}^{(j)})-\max_{i\in G}z_{\boldsymbol{\delta}^{*}}(\boldsymbol{x}_{p}^{(i)})\ge3C_{\mathrm{lip}}\rho,$
	then $\nabla_{\boldsymbol{\delta}}\mathcal{L}(\hat{\bolddelta},\hat{\boldw})-\nabla_{\boldsymbol{\delta}}\mathcal{L}(\hat{\bolddelta},\boldsymbol{w}^{*})=0$. 
\end{lem}
\begin{proof}
	\begin{align}
	\min_{j\in B}z_{\boldsymbol{\delta}^{*}}(\boldsymbol{x}_{p}^{(j)})-\max_{i\in G}z_{\boldsymbol{\delta}^{*}}(\boldsymbol{x}_{p}^{(i)}) & =\min_{j\in B}\hat{z}_{\boldsymbol{\delta}^{*}}(\boldsymbol{x}^{(j)})-\max_{i\in G}\hat{z}_{\boldsymbol{\delta}^{*}}(\boldsymbol{x}^{(i)})\ge3C_{\mathrm{lip}}\rho\label{eq:wellsep}
	\end{align}
	Due to Assumption \ref{ass.ratio.lip-1} and \eqref{eq:wellsep},
	\begin{align}
	\forall i & \in G,j\in B,\text{and }\boldu\in\mathrm{Ball}(\rho),\;\hat{z}{}_{\boldsymbol{\delta}^{*}+\boldsymbol{u}}(\boldx^{(j)})>\hat{z}{}_{\boldsymbol{\delta}^{*}+\boldsymbol{u}}(\boldx^{(i)}).\label{eq:ranking}
	\end{align}
	According to the optimality condition of \eqref{eq:min_max}, we should
	simply assign non-zero weights $w_{i}$ to the $\nu n_{p}$ samples
	corresponding to the smallest $\hat{z}{}_{\boldsymbol{\delta}^{*}+\boldsymbol{u}}$
	values. Therefore, from \eqref{eq:ranking} we can see that $\hat{G}=G$.
	Moreover, since the inequality of \eqref{eq:ranking} holds strictly
	and $\nu=\frac{|G|}{n_{p}}=\frac{|\hat{G}|}{n_{p}}$, all weights
	must be set to $\frac{1}{n_{p}}$ in order to minimize the inner problem
	of \eqref{eq:min_max}, i.e., $\forall i\in G$, $\hat{w}_i = \frac{1}{n_p}$ and $\forall i\in B$, $\hat{w}_i = 0$.
\end{proof}
Now we can set $\kappa_{2}=0,\tau_{2}(n,d)=0$ to make \eqref{eq:ass2} hold.

As explained in Section \eqref{subsec:Preparations}, we need to confirm$\|\nabla_{\boldsymbol{\delta}}\mathcal{L}(\boldsymbol{\delta}^{*},\boldsymbol{w}^{*})\|_{\infty}$
converges to $\text{0}$ as the sample size goes to inifinity where
$\nabla_{\boldsymbol{\delta}}\mathcal{L}(\boldsymbol{\delta}^{*},\boldsymbol{w}^{*})=\frac{1}{n_{p}}\sum_{i\in G}\nabla_{\boldsymbol{\delta}}\hat{z}_{\boldsymbol{\delta}^{*}}(\boldsymbol{x}_{p}^{(i)}).$
Since 
\begin{align*}
\|\frac{1}{n_{p}}\sum_{i\in G}\nabla_{\boldsymbol{\delta}}\hat{z}_{\boldsymbol{\delta}^{*}}(\boldsymbol{x}_{p}^{(i)})\|_{\infty}\le\frac{1}{\nu}\cdot\|\frac{1}{n_{p}}\sum_{i\in G}\nabla_{\boldsymbol{\delta}}\hat{z}_{\boldsymbol{\delta}^{*}}(\boldsymbol{x}_{p}^{(i)})\|_{\infty}=\|\frac{1}{|G|}\sum_{i\in G}\nabla_{\boldsymbol{\delta}}\hat{z}_{\boldsymbol{\delta}^{*}}(\boldsymbol{x}_{p}^{(i)})\|_{\infty},
\end{align*}
we only need to bound $\left\Vert \frac{1}{|G|}\sum_{i\in G}\nabla_{\boldsymbol{\delta}}\hat{z}_{\boldsymbol{\delta}^{*}}(\boldsymbol{x}_{p}^{(i)})\right\Vert _{\infty}$.
As samples in $G$ are i.i.d. samples drawn from $P$, here can we invoke
the Lemma 2 from \cite{Liu2016a}. 
First we need the following conditions:
\begin{assumption}
	\label{ass:lemma21}
	For any vector $\boldu \in \mathbb{R}^{\text{dim}(\bolddelta^*)}$ such that $\bolddelta^* + \boldu\in \mathrm{Ball}(\rho)$, the Hessian of the likelihood function, $\nabla^2 \mathcal{L}(\bolddelta^*+\boldu)$, has a bounded spectral norm, i.e., 
	$
		\|\nabla^2 \mathcal{L}(\bolddelta^* + \boldu)\| \le \lambda_\mathrm{max}.
	$
\end{assumption}
\begin{assumption}[Smooth Density Ratio Model Assumption]
	\label{ass:lemma22}
	For any vector $\boldu \in \mathbb{R}^{\text{dim}(\bolddelta^*)}$ such that $\bolddelta^* + \boldu\in \mathrm{Ball}(\rho)$ and every $a\in \mathbb{R}$, the following inequality holds:
	\begin{align*}
	\mathbb{E}_q \left[\exp\left( a\left( r(\boldx, \bolddelta^* + \boldu) - 1 \right)\right) \right]  \le \exp\left(Ka^2\right).
	\end{align*}
\end{assumption}
If $n_{q}=\Omega(|G|^{2})$, and $\lambda_n \ge  \sqrt{\frac{K_1\log d}{|G|}}$, according to Lemma 2 from \cite{Liu2016a} we have 
\begin{align}
\label{eq:nablarate}
P\left(\left\Vert \frac{1}{|G|}\sum_{i\in G}\nabla_{\boldsymbol{\delta}}\hat{z}_{\boldsymbol{\delta}^{*}}(\boldsymbol{x}_{p}^{(i)})\right\Vert {}_{\infty} \ge \lambda_n \right) \le \exp\left( - c_1 |G| \right),
\end{align}
where $K_1$ and $c_1$ are constants. 
Finally, we can re-state the Theorem \ref{thm:main} using $\kappa_{2} = 0$, $\tau_{2} = 0$ and \eqref{eq:nablarate} to obtain Theorem \ref{cor:outlier}. 

\subsection{Proof of Theorem \ref{cor:truncationsetting}}
\label{sec:proofoftruncation}
First we verify \eqref{eq:ass2}. 
\begin{lem}
	\label{lem:boundeddiff2}Under Assumptions \ref{ass.ratio.lip} and
	\ref{assu:regP}, 
	\begin{align*}
	\|\nabla_{\boldsymbol{\delta}}\mathcal{L}(\hat{\boldsymbol{\delta}},\hat{\boldw})-\nabla_{\boldsymbol{\delta}}\mathcal{L}(\hat{\bolddelta},\boldsymbol{w}^{*})\|_{\infty}\le2C_{\mathrm{CDF}}\cdot\|\boldsymbol{u}\|C_{p}+\frac{2L\cdot C_{p}}{\sqrt{n_{p}}},
	\end{align*}
	where $L$ is a positive constant. The second term reflects the cost
	of using the empirical sample to control the $\nu$-th quantile in
	\eqref{eq:ranking}. 
\end{lem}
Therefore 
\begin{align*}
\langle\nabla_{\boldsymbol{\delta}}\mathcal{L}(\hat{\bolddelta},\hat{\boldw})-\nabla_{\boldsymbol{\delta}}\mathcal{L}(\hat{\bolddelta},\boldsymbol{w}^{*}),\boldsymbol{u}\rangle\ge & -\left(2C_{\mathrm{CDF}}\cdot\|\boldsymbol{u}\|C_{p}+\frac{2L\cdot C_{p}}{\sqrt{n_{p}}}\right)\|\boldsymbol{u}\|^2_{1}\\
\ge & -2\sqrt{k}C_{\mathrm{CDF}}C_{p}\|\boldsymbol{u}\|^{2}-\frac{2L\cdot C_{p}\|\boldsymbol{u}\|_{1}}{\sqrt{n_{p}}}.
\end{align*}
It can be seen that $\kappa_{2}=2\sqrt{k}C_{\mathrm{CDF}}C_{p},\tau_{2}(n,d)=\frac{2L\cdot C_{p}}{\sqrt{n_{p}}}.$
The proof of Lemma \ref{lem:boundeddiff2} uses a fact that only $\boldsymbol{x}_{p}$
in the ``zone'' $T(\boldsymbol{u},\frac{L_{1}}{\sqrt{n_{p}}})$
are ``dangerous'' as they may be mistakenly included or missed out
under small perturbation of $\boldsymbol{u}$. See Section \ref{subsec:lemma3}
in Appendix for the proof.

To show $\|\nabla_{\boldsymbol{\delta}}\mathcal{L}(\boldsymbol{\delta}^{*},\boldsymbol{w}^{*})\|_\infty\rightarrow0,$
we need some extra procedures since $z_{\boldsymbol{\delta}^{*}}(\boldsymbol{x}_{q})$
are not necessarily upper-bounded by $t(\boldsymbol{\delta}^{*})$. 
The following lemma bounds $\|\nabla_{\boldsymbol{\delta}}\mathcal{L}(\boldsymbol{\delta}^{*},\boldsymbol{w}^{*})\|_{\infty}$. 
\begin{lem}
	\label{lem:supnorm}Under Assumptions \ref{assu:.boundedrq}, \ref{assu:regP}, \ref{ass:lemma21} and \ref{ass:lemma22} holds, and if 
	\begin{align}
	\label{eq:nablarate2}
	\lambda_n \ge \sqrt{\frac{K_1'\log d}{|\overline{X}{}^{p}(\boldsymbol{\delta}^{*})|}}+\frac{2C_{r}^{2}C_{q}|X_{q}\backslash\overline{X}^{q}(\boldsymbol{\delta}^{*})|}{n_{q}}
	\end{align}	
	$\|\nabla_{\boldsymbol{\delta}}\mathcal{L}(\boldsymbol{\delta}^{*},\boldsymbol{w}^{*})\|_{\infty} \le \lambda_n$ with probability at least $1 - \exp(c'_1 |\overline{X}{}^{p}(\boldsymbol{\delta}^{*})|))$, where $c'_1$ and $K'_1$ are constants, 
\end{lem}
See Section \ref{subsec:lemmasupnorm} in Appendix for the proof.

Finally, we can restate Theorem \ref{thm:main} as Theorem \ref{cor:truncationsetting} using $\kappa_{1}=\frac{\nu\kappa'_{1}}{2C_{r}^{2}}, \tau_{1}(n,d)=\frac{\nu c}{2C_{r}^{2}\sqrt{n_{q}}}$, $\kappa_{2}=2\sqrt{k}C_{\mathrm{CDF}}C_{p},\tau_{2}(n,d)=\frac{2L\cdot C_{p}}{\sqrt{n_{p}}}$ and \eqref{eq:nablarate2}, making sure that $\kappa_1>\kappa_2$.

\subsection{Proof of Lemma\label{subsec:lemmasupnorm} \ref{lem:supnorm}}
\begin{figure}
\centering{}\includegraphics[width=0.6\textwidth]{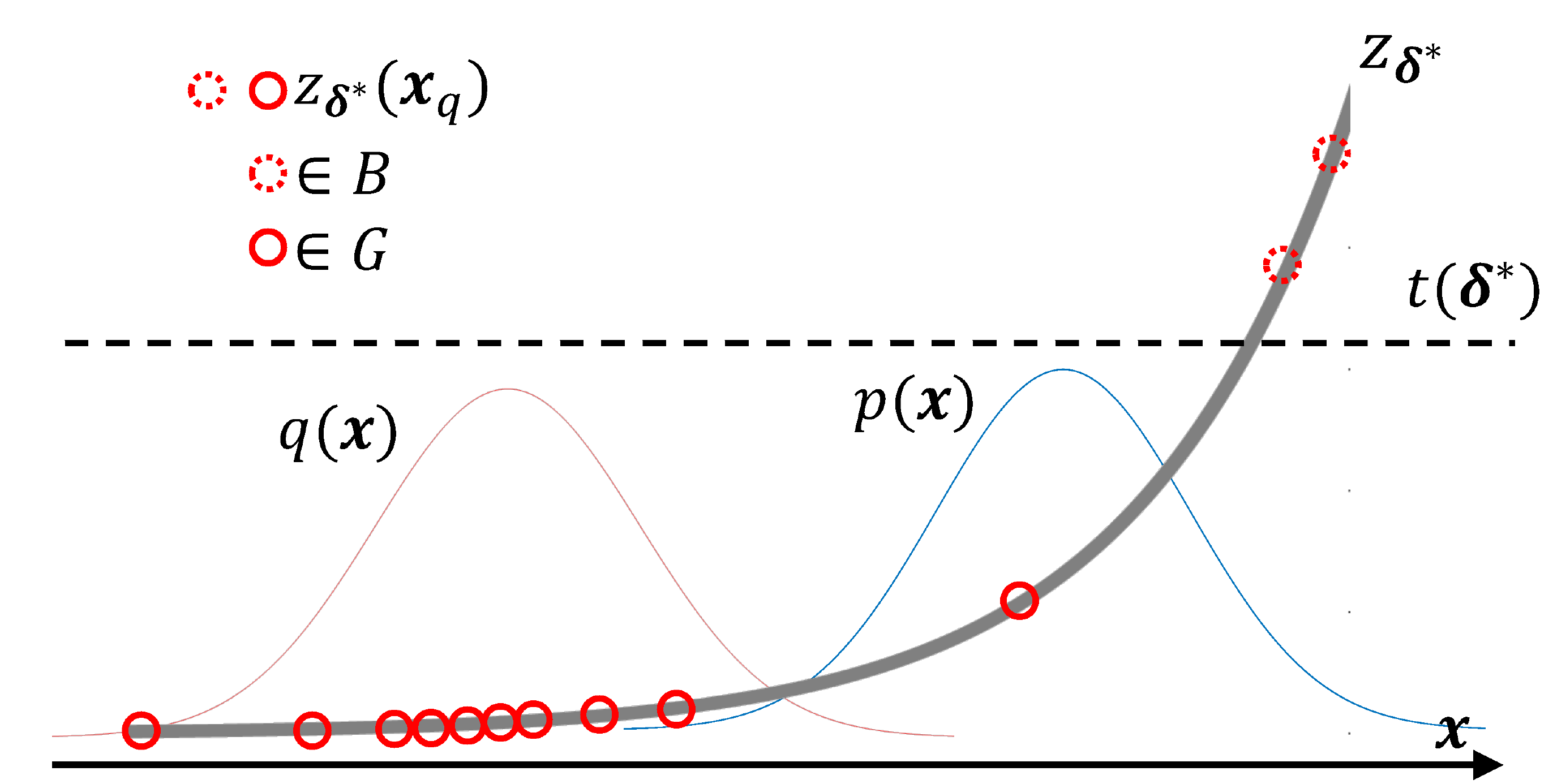}\caption{An illustration of $B$ and $G$ in the case of truncation setting.
In this setting, we treat $X_{q}\backslash\overline{X}{}^{q}(\boldsymbol{\delta}^{*})$
as a kind of outlier of $Q$ and only appear in very small quantity.
\label{fig:BandG} }
\end{figure}
First, we recycle some notations from the previous section: $G:=\overline{X}^{q}(\boldsymbol{\delta}^{*}),B:=X_{q}\backslash\overline{X}^{q}(\boldsymbol{\delta}^{*}).$
The reason for this arrangement can be seen from Figure \ref{fig:BandG}.

Denote $e^{(j)}:=\exp\left[\langle\boldsymbol{\delta}^{*},\boldsymbol{x}_q^{(j)}\rangle\right],s:=\sum_{j=1}^{n_{q}}e^{(j)}$
and $\bar{s}=\sum_{i\in G}e^{(i)}$. Note that 
\[
\nabla_{\boldsymbol{\delta}}\mathcal{L}(\boldsymbol{\delta}^{*},\boldsymbol{w}^{*})=\frac{1}{n_{p}}\sum_{i\in\overline{X}{}^{p}(\boldsymbol{\delta}^{*})}\left[\boldsymbol{x}_{p}^{(i)}-\nabla_{\boldsymbol{\text{\ensuremath{\delta}}}}\log\widehat{N}(\boldsymbol{\delta}^*)\right].
\]

\begin{align*}
 \|\nabla_{\boldsymbol{\delta}}\mathcal{L}(\boldsymbol{\delta}^{*},\boldsymbol{w}^{*})\|_{\infty}
= & \|\frac{1}{n_{p}}\sum_{i\in\overline{X}{}^{p}(\boldsymbol{\delta}^{*})}\left[\boldsymbol{x}_{p}^{(i)}-\nabla_{\boldsymbol{\delta}}\log\widehat{N}(\boldsymbol{\delta}^*)\right]\|_{\infty}\\
= & \frac{1}{n_{p}}\|\sum_{i\in\overline{X}{}^{p}(\boldsymbol{\delta}^{*})}\left[\boldsymbol{x}_{p}^{(i)}-\sum_{j=1}^{n_{q}}\frac{e^{(j)}}{s}\boldsymbol{x}_{q}^{(j)}\right]\|_{\infty}\\
= & \frac{1}{n_{p}}\|\sum_{i\in\overline{X}{}^{p}(\boldsymbol{\delta}^{*})}\left[\boldsymbol{x}_{p}^{(i)}-\sum_{j\in G}\frac{e^{(j)}}{s}\boldsymbol{x}_{q}^{(j)}-\sum_{j\in B}\frac{e^{(j)}}{s}\boldsymbol{x}_{q}^{(j)}\right]\|_{\infty}\\
= & \frac{1}{n_{p}}\|\sum_{i\in\overline{X}{}^{p}(\boldsymbol{\delta}^{*})}\left[\boldsymbol{x}_{p}^{(i)}-\frac{\bar{s}}{s}\sum_{j\in G}\frac{e^{(j)}}{\bar{s}}\boldsymbol{x}_{q}^{(j)}-\sum_{j\in B}\frac{e^{(j)}}{s}\boldsymbol{x}_{q}^{(j)}\right]\|_{\infty}\\
= & \frac{1}{n_{p}}\|\sum_{i\in\overline{X}{}^{p}(\boldsymbol{\delta}^{*})}\left[\boldsymbol{x}_{p}^{(i)}-\sum_{j\in G}\frac{e^{(j)}}{\bar{s}}\boldsymbol{x}_q^{(j)}+(1-\frac{\bar{s}}{s})\sum_{j\in G}\frac{e^{(j)}}{\bar{s}}\boldsymbol{x}_q^{(j)}-\sum_{j\in B}\frac{e^{(j)}}{s}\boldsymbol{x}_q^{(j)}\right]\|_{\infty}\\
\le & \underbrace{\frac{1}{n_{p}}\|\sum_{i\in\overline{X}{}^{p}(\boldsymbol{\delta}^{*})}\boldsymbol{x}_{p}^{(i)}-\sum_{j\in G}\frac{e^{(j)}}{\bar{s}}\boldsymbol{x}_q^{(j)}||_{\infty}}_{a(n,d)}
+||(1-\frac{\bar{s}}{s})\sum_{j\in G}\frac{e^{(j)}}{\bar{s}}\boldsymbol{x}_q^{(j)}-\sum_{j\in B}\frac{e^{(j)}}{s}\boldsymbol{x}_q^{(j)}\|_{\infty}\\
\le & a(n,d)+\frac{s-\bar{s}}{s}\sum_{j\in G}\|\frac{e^{(j)}}{\bar{s}}\boldsymbol{x}_q^{(j)}\|_{\infty}+\sum_{j\in B}\|\frac{e^{(j)}}{s}\boldsymbol{x}_q^{(j)}\|_{\infty}\\
\le & a(n,d)+\frac{C_{r}^{2}|B|}{n_{q}}\cdot\frac{1}{|G|}\sum_{j\in G}\|\boldsymbol{x}^{(j)}\|_{\infty}+\frac{C_{r}}{n_{q}}\sum_{j\in B}\|\boldsymbol{x}_q^{(j)}\|_{\infty}\\
\le & a(n,d)+\frac{C_{r}^{2}|B|C_{q}}{n_{q}}+\frac{C_{r}|B|C_{q}}{n_{q}}\\
\le & a(n,d)+\frac{2C_{r}^{2}|B|C_{q}}{n_{q}}
\end{align*}
Now, as $\overline{X}^p(\bolddelta^*)$ and $G$ contains only i.i.d. samples and due to the definition of $\bolddelta^*$, we can invoke Lemma 2 again from \cite{Liu2016a} to bound $a(n,d)$. That is if Assumptions \ref{ass:lemma21} and \ref{ass:lemma22} hold and $n_q = \Omega(n_p^2)$, and $\lambda_n \ge \sqrt{\frac{K'_1 \log d}{|\overline{X}{}^{p}(\boldsymbol{\delta}^{*})|}}$
\begin{align}
\label{eq:nablarate3}
P\left(a(n,d) \ge \lambda_n \right) \le \exp\left( - c'_1 |\overline{X}{}^{p}(\boldsymbol{\delta}^{*})| \right),
\end{align}
where $K'_1$ and $c'_1$ are constants. By taking the extra $\frac{2C_{r}^{2}|B|C_{q}}{n_{q}}$ into account, we obtain Lemma \ref{lem:supnorm}. 

\subsection{Proof of Lemma \label{subsec:lemma3} \ref{lem:boundeddiff2}}

Before we start, we need to define a few empirical counterparts of
population quantities used in Section \ref{subsec:trimmedsetting}.
\begin{itemize}
\item $P_{n}$ is the empirical distribution of $P$.
\item $\hat{t}(\boldsymbol{\delta})$ is the empirical version of $t(\boldsymbol{\delta})$
and is defined according to 
\[
P_{n_{p}}\left[\hat{z}{}_{\boldsymbol{\delta}}<\hat{t}{}_{\nu}(\boldsymbol{\delta}))|X_{q}\right]\le\nu,\:\:\:P_{n_{p}}\left[\hat{z}{}_{\boldsymbol{\delta}}\le\hat{t}{}_{\nu}(\boldsymbol{\delta}))|X_{q}\right]\ge\nu
\]
\item The set $\overline{X}_{n}(\boldsymbol{\delta})$ is similar to $\overline{X}(\boldsymbol{\delta})$ but defined by $\hat{z}$ and $\hat{t}$: 
\[
\overline{X}_{n}(\boldsymbol{\delta}):=\left\{ \boldsymbol{x}\in\mathbb{R}^{d}|\hat{z}{}_{\boldsymbol{\delta}}(\boldsymbol{x})<\hat{t}(\boldsymbol{\delta})\right\} .
\]
\item $\overline{X}_{n}^{p}(\boldsymbol{\delta}):=X_{p}\cap\overline{X}_{n}(\boldsymbol{\delta})$. 
\item The ``borderline points'' of $X_{p}$: $X_{\mathrm{border}}(\boldsymbol{\delta}):=\{\boldsymbol{x}\in X_{p}|\hat{z}{}_{\boldsymbol{\delta}}(\boldsymbol{x})=\hat{t}{}_{\nu}(\boldsymbol{\boldsymbol{\delta}}))\}$. 
\end{itemize}
\begin{proof}
We first expand $\|\nabla_{\boldsymbol{\delta}}\mathcal{L}(\hat{\boldsymbol{\delta}},\hat{\boldsymbol{w}})-\nabla_{\boldsymbol{\delta}}\mathcal{L}(\hat{\boldsymbol{\delta}},\boldsymbol{w}^{*})\|_{\infty}$
as 
\begin{align}
 & \|\nabla_{\boldsymbol{\delta}}\mathcal{L}(\boldsymbol{\delta}^{*}+\hat{\boldsymbol{u}},\hat{\boldsymbol{w}})-\nabla_{\boldsymbol{\delta}}\mathcal{L}(\boldsymbol{\delta}^{*}+\hat{\boldsymbol{u}},\boldsymbol{w}^{*})\|_{\infty}\nonumber \\
= & \|\sum_{i,w_{i}\neq0}\hat{w}_{i}\boldsymbol{x}_{p}^{(i)}-\frac{1}{n_{p}}\sum_{i\in\overline{X}^{p}(\boldsymbol{\delta}^{*})}\boldsymbol{x}_{p}^{(i)}\|_{\infty}\nonumber \\
\le & \frac{1}{n_{p}}\sum_{i\in\underbrace{\overline{X}_{n}^{p}(\boldsymbol{\delta}^{*}+\hat{\boldsymbol{u}})\backslash\overline{X}^{p}(\boldsymbol{\delta}^{*})}_{M_{1}(\hat{\boldsymbol{u}})}}\|\boldsymbol{x}_{p}^{(i)}\|_{\infty}+\frac{1}{n_{p}}\sum_{i\in\underbrace{\overline{X}^{p}(\boldsymbol{\delta}^{*})\backslash\overline{X}_{n}^{p}(\boldsymbol{\delta}^{*}+\hat{\boldsymbol{u}})}_{M_{2}(\hat{\boldsymbol{u}})}}\|\boldsymbol{x}_{p}^{(i)}\|_{\infty}\nonumber \\
 & \qquad\qquad\qquad\qquad\qquad\qquad\qquad\;\;\:\:+\frac{1}{n_{p}}\sum_{i\in X_{\mathrm{border}}(\boldsymbol{\delta}^{*}+\hat{\boldsymbol{u}})}\|\boldsymbol{x}_{p}^{(i)}\|_{\infty}\nonumber \\
= & \frac{1}{n_{p}}\sum_{i\in M(\hat{\boldu})}\|\boldsymbol{x}_{p}^{(i)}\|_{\infty}+\frac{1}{n_{p}}\sum_{i\in X_{\mathrm{border}}(\boldsymbol{\delta}^{*}+\hat{\boldsymbol{u}})}\|\boldsymbol{x}_{p}^{(i)}\|_{\infty},\label{eq:aa}
\end{align}
where $M(\boldu):=M_{1}(\boldu)\cup M_{2}(\boldu),\text{ given }\boldu\in\mathrm{Ball}(\rho).$
Note we isolate the borderline points $X_{\mathrm{border}}$ in our
analysis as they may have interior weights, i.e., $w_{i}\text{\ensuremath{\in}}[0,\frac{1}{n_{p}}].$

We first figure out the cardinality of $M(\boldu)$, a set where samples
are \emph{likely }to be ``misplaced'' to the other set under a small
perturbation. However, direct quantifying $M(\boldu)$ is hard but
we now show that $M(\boldu)\subseteq X_{p}\cap T(\boldsymbol{u},\epsilon)$
whose cardinality is bounded by our assumptions. See Figure \ref{fig:The-relationship-of}
for details.  
\begin{figure}
\centering{}\includegraphics[width=0.7\textwidth]{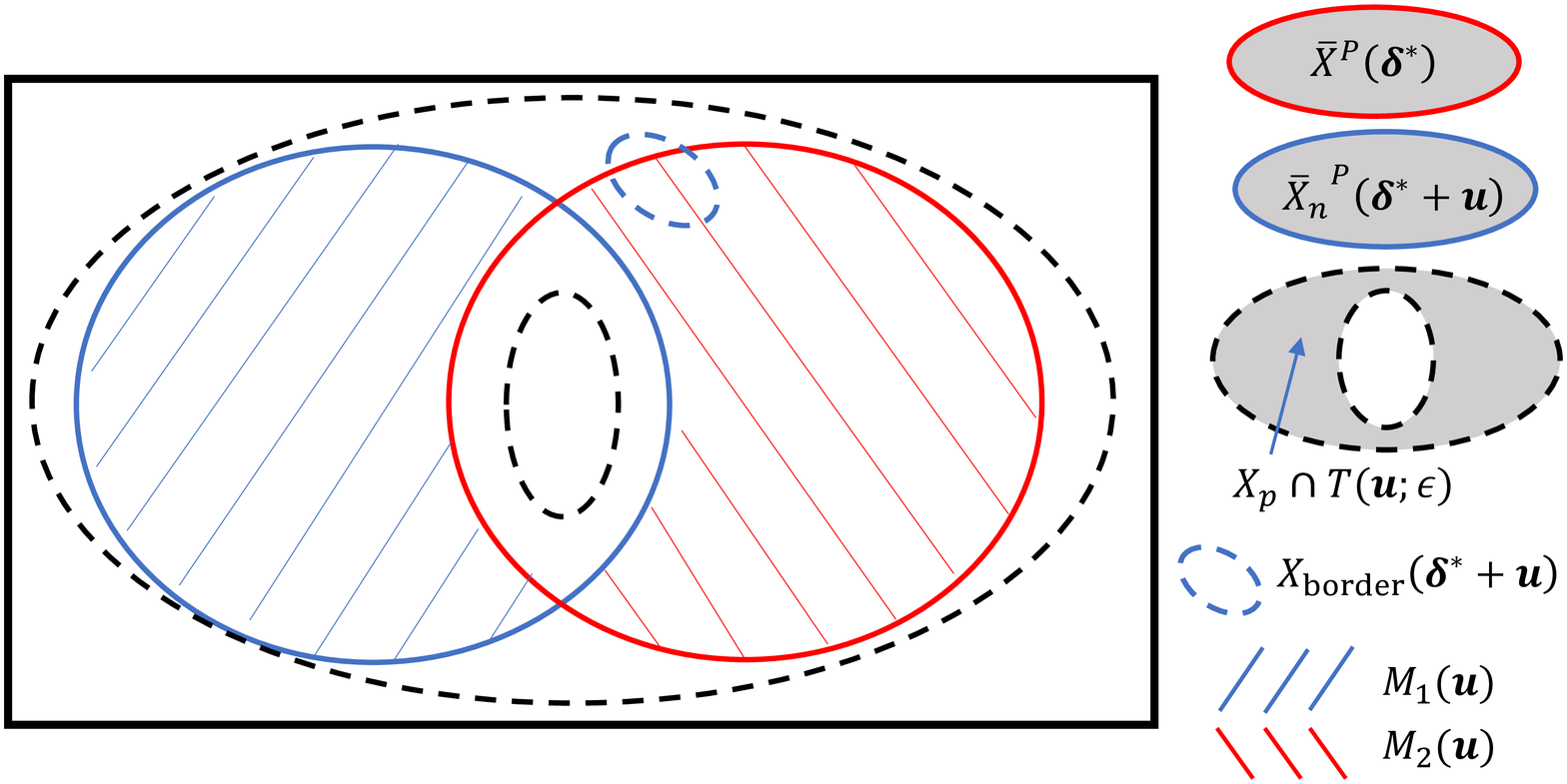}\caption{The relationship of $\overline{X}^{P}(\boldsymbol{\delta}^{*})$,
$\overline{X}_{n}^{P}(\boldsymbol{\delta}^{*}+\boldsymbol{u})$, $X_{p}\cap T(\boldsymbol{u},\epsilon)$
and $X_{\mathrm{border}}(\boldsymbol{\delta}^{*}+\boldsymbol{u})$.\label{fig:The-relationship-of}}
\end{figure}

First, we show that if $z_{\boldsymbol{\delta}^{*}}(\boldsymbol{x}_{p})\ge t(\boldsymbol{\delta}^{*})+2C_{\mathrm{lip}}||\boldsymbol{u}||+\epsilon,$
then $\boldsymbol{x}_{p}\notin\overline{X}^{p}(\boldsymbol{\delta}^{*})\cup\overline{X}_{n}^{p}(\boldsymbol{\delta}^{*}+\boldsymbol{u})$.As we will see, $\epsilon \in (0,1)$ is chosen afterwards.

Under this setting, obviously, $\boldsymbol{x}_{p}\notin\overline{X}^{p}(\boldsymbol{\delta}^{*})$, thus it is suffice to show that $\boldsymbol{x}_{p}\notin\overline{X}_n^{p}(\boldsymbol{\delta}^{*} + \boldu)$. 
Note that for any constant c, the quantile of $z':=z(\bolddelta^*)+c $ is $t(\bolddelta^*)+c$. 

Since $z_{\boldsymbol{\delta}^{*}}$
and $\hat{z}_{\boldsymbol{\delta}^{*}}$ differ only by their normalization
functions, we have $z_{\boldsymbol{\delta}^{*}}(\boldsymbol{x}_{p})-t(\boldsymbol{\delta}^{*})=\hat{z}_{\boldsymbol{\delta}^{*}}(\boldsymbol{x}_{p})-t'(\boldsymbol{\delta}^{*})$,
where $t'(\boldsymbol{\delta}^{*})$ is defined as $P\left[\hat{z}{}_{\boldsymbol{\delta}^*}< t'{}_{\nu}(\boldsymbol{\delta}^*))|X_{q}\right]\le\nu$ and 
$P\left[\hat{z}{}_{\boldsymbol{\delta}^*}\le t'{}_{\nu}(\boldsymbol{\delta}^*))|X_{q}\right]\ge\nu$
for a given $X_{q}$, so we have $\hat{z}_{\boldsymbol{\delta}^{*}}(\boldsymbol{x}_{p})\ge t'(\boldsymbol{\delta}^{*})+2C_{\mathrm{lip}}||\boldsymbol{u}||+\epsilon.$
Combining this inequality with Assumption \ref{ass.ratio.lip}, we
have 
\begin{equation}
\hat{z}_{\boldsymbol{\delta}^{*}+\boldsymbol{u}}(\boldsymbol{x}_{p})\ge\hat{z}_{\boldsymbol{\delta}^{*}}(\boldsymbol{x}_{p})-C_{\mathrm{lip}}||\boldsymbol{u}||\ge t'(\boldsymbol{\delta}^{*})+C_{\mathrm{lip}}||\boldsymbol{u}||+\epsilon\label{eq:bb}
\end{equation}
From Dvoretzky–Kiefer–Wolfowitz inequality if $n_{p}$ is large enough,
with high probability $\left|t'(\boldsymbol{\delta}^{*})-\hat{t}(\boldsymbol{\delta}^{*})\right|\le\frac{L_{1}}{\sqrt{n_{p}}}\le1$
which is independent of the choice of $X_{q}$. Thus we set $\epsilon=\frac{L_{1}}{\sqrt{n_{p}}}$,
and 
\begin{equation}
t'(\boldsymbol{\delta}^{*})+\frac{L_{1}}{\sqrt{n_{p}}}+C_{\mathrm{lip}}||\boldsymbol{u}\|\ge\hat{t}(\boldsymbol{\delta}^{*})+C_{\mathrm{lip}}||\boldsymbol{u}||\text{ w.h.p.}\label{eq:cc}
\end{equation}
From Assumption \ref{ass.ratio.lip}, $\hat{z}_{\boldsymbol{\delta}^{*}+\boldsymbol{u}}$
and $\hat{z}_{\boldsymbol{\delta}^{*}}$ differ only by $C_{\mathrm{lip}}\|\boldsymbol{u}\|$,
which means their $\nu$-percentile $\hat{t}(\boldsymbol{\delta}^{*}+\boldsymbol{u})$
and $\hat{t}(\boldsymbol{\delta}^{*})$ differ by $C_{\mathrm{lip}}\|\boldsymbol{u}\|$
at most. Thus, 
\begin{equation}
\hat{t}(\boldsymbol{\delta}^{*})+C_{\mathrm{lip}}||\boldsymbol{u}||\ge\hat{t}(\boldsymbol{\delta}^{*}+\boldsymbol{u})\label{eq:dd}
\end{equation}
From \eqref{eq:bb} \eqref{eq:cc} and \eqref{eq:dd}, we now have
$\hat{z}_{\boldsymbol{\delta}^{*}+\boldsymbol{u}}(\boldsymbol{x}_{p})\ge\hat{t}(\boldsymbol{\delta}^{*}+\boldsymbol{u})$
which means 
\begin{align*}
\boldsymbol{x}_{p}\notin\overline{X}_{n}^{p}(\boldsymbol{\delta}^{*}+\boldsymbol{u})
\end{align*}
with high probability. As we have mentioned earlier, it is obvious
that $\boldsymbol{x}_{p}\notin\overline{X}^{p}(\boldsymbol{\delta}^{*})$,
so 
\begin{align}
\boldsymbol{x}_{p}\notin\overline{X}^{p}(\boldsymbol{\delta}^{*})\cup\overline{X}_{n}^{p}(\boldsymbol{\delta}^{*}+\boldsymbol{u}).\label{eq:b1}
\end{align}
Similarly, one can show if $z_{\boldsymbol{\delta}^{*}}(\boldsymbol{x}_{p})\le t(\boldsymbol{\delta}^{*})-2C_{\mathrm{lip}}||\boldsymbol{u}||-\epsilon$,
then 
\begin{align}
\boldsymbol{x}_{p}\in\overline{X}_{n}^{p}(\boldsymbol{\delta}^{*}+\boldsymbol{u})\cap\overline{X}^{p}(\boldsymbol{\delta}^{*})\label{eq:b2}
\end{align}
(which is the center-most region in Figure \ref{fig:The-relationship-of})
with high probability. Now we can conclude that:
\begin{align}
M(\boldu)\subseteq X_{p}\cap T(\boldsymbol{u},\frac{L_{1}}{\sqrt{n_{p}}})\text{ w.h.p}.\label{eq:ee}
\end{align}
Due to Dvoretzky–Kiefer–Wolfowitz inequality, 
\begin{align*}
P_{n_{p}}(\boldsymbol{x}_{p} & \in T(\boldsymbol{u},\frac{L_{1}}{\sqrt{n_{p}}}))-P(\boldsymbol{x}_{p}\in T(\boldsymbol{u},\frac{L_{1}}{\sqrt{n_{p}}}))\le\frac{L_{2}}{\sqrt{n_{p}}}
\end{align*}
holds with probability at least $\exp\left[-2L_{2}^{2}\right],\forall L_{2}>0$.
Thus, using Assumption \ref{ass.ratio.lip} we have 
\begin{align*}
P_{n_{p}}(\boldsymbol{x}_{p} & \in T(\boldsymbol{u},\frac{L_{1}}{\sqrt{n_{p}}}))\le C_{\mathrm{CDF}}\cdot\|\boldsymbol{u}\|+\frac{L_{1}}{\sqrt{n_{p}}}+\frac{L_{2}}{\sqrt{n_{p}}}\text{ w.h.p}.
\end{align*}
Now we know the cardinality of $X_{p}\cap T(\boldsymbol{u},\frac{L_{1}}{\sqrt{n_{p}}})$
can be bounded by $\left(C_{\mathrm{CDF}}\cdot\|\boldsymbol{u}\|+\frac{L_{1}+L_{2}}{\sqrt{n_{p}}}\right)\cdot n_{p}$
with high probability. Finally, we have 
\begin{align}
\frac{1}{n_{p}}\sum_{i\in X_{p}\cap T(\boldsymbol{u},\frac{L_{1}}{\sqrt{n_{p}}})}\|\boldx_{p}^{(i)}\|_{\infty} & \le\frac{1}{n_{p}}\left(C_{\mathrm{CDF}}\cdot\|\boldsymbol{u}\|+\frac{L_{1}+L_{2}}{\sqrt{n_{p}}}\right)\cdot n_{p}C_{p}\nonumber \\
 & \le C_{\mathrm{CDF}}\cdot\|\boldsymbol{u}\|C_{p}+\frac{(L_{1}+L_{2})\cdot C_{p}}{\sqrt{n_{p}}}\label{eq:ff}
\end{align}

Now, we show $X_{\mathrm{border}}(\boldsymbol{\delta}^{*}+\boldsymbol{u})\subseteq X_{p}\cap T(\boldsymbol{u},\frac{L_{1}}{\sqrt{n_{p}}})$.
The proof for this is similar to the arguments above. Using Assumption \ref{ass.ratio.lip}, it can be shown
that 
\[
\hat{t}(\boldsymbol{\delta}^{*}+\boldsymbol{u})\in\left[t'(\boldsymbol{\delta}^{*})-C_{\mathrm{lip}}\|\boldsymbol{u}\|-\frac{L_{1}}{\sqrt{n_{p}}},t'(\boldsymbol{\delta}^{*})+C_{\mathrm{lip}}\|\boldsymbol{u}\|+\frac{L_{1}}{\sqrt{n_{p}}}\right],
\]
 and from definition, $\forall\boldsymbol{x}\in X_{\mathrm{border}}(\boldsymbol{\delta}^{*}+\boldsymbol{u}),\hat{z}_{\boldsymbol{\delta}^{*}+\boldsymbol{u}}(\boldsymbol{x})=\hat{t}(\boldsymbol{\delta}^{*}+\boldsymbol{u})$,
\[
\hat{z}_{\boldsymbol{\delta}^{*}+\boldsymbol{u}}(\boldsymbol{x})\in\left[t'(\boldsymbol{\delta}^{*})-C_{\mathrm{lip}}\|\boldsymbol{u}\|-\frac{L_{1}}{\sqrt{n_{p}}},t'(\boldsymbol{\delta}^{*})+C_{\mathrm{lip}}\|\boldsymbol{u}\|+\frac{L_{1}}{\sqrt{n_{p}}}\right],
\]
and due to Assumption \ref{ass.ratio.lip}, 
\[
\hat{z}_{\boldsymbol{\delta}^{*}}(\boldsymbol{x})\in\left[t'(\boldsymbol{\delta}^{*})-2C_{\mathrm{lip}}\|\boldsymbol{u}\|-\frac{L_{1}}{\sqrt{n_{p}}},t'(\boldsymbol{\delta}^{*})+2C_{\mathrm{lip}}\|\boldsymbol{u}\|+\frac{L_{1}}{\sqrt{n_{p}}}\right].
\]
Again, this relationship does not change if we replace $\hat{z}$
and $t'$ at the same time with $z$ and $t$ 
\begin{align}
z_{\boldsymbol{\delta}^{*}}(\boldsymbol{x})\in\left[t(\boldsymbol{\delta}^{*})-2C_{\mathrm{lip}}\|\boldsymbol{u}\|-\frac{L_{1}}{\sqrt{n_{p}}},t(\boldsymbol{\delta}^{*})+2C_{\mathrm{lip}}\|\boldsymbol{u}\|+\frac{L_{1}}{\sqrt{n_{p}}}\right] & \subseteq T(\boldsymbol{u},\frac{L_{1}}{\sqrt{n_{p}}}).\label{eq:gg}
\end{align}

Inequalities \eqref{eq:aa}, \eqref{eq:ee}, \eqref{eq:ff} and \eqref{eq:gg}
complete the proof. 
\end{proof}

\section{Numerical Analysis\label{sec:Numerical-Analysis}}

In this section, we present a few numerical experimental results under
outlier and truncation setting. In all experiments, we set $n_{p}=n_{q}=5000,\lambda=0,$
and the solution of $\hat{\boldsymbol{\delta}}$ was obtained using
Algorithm \ref{alg}. We let $f(x)=x.$ Note this is the correct log-ratio
model for two Gaussian distributions with different means.

\paragraph{Outlier Setting}

In this setting, we first generate two ``good'' datasets $G\iid p(x)=N(0,1),$
and $X_{q}\iid q(x)=N(-.75,1)$. The outlier set $B_{b}$ is generated
from a uniform distribution $U(-0.4+b,0.4+b),b\in[0,6]$. The density
ratio estimation is performed using two sets of data: $X_{p,b}=\left\{ G,B_{b}\right\} $
and $X_{q}$, where the cardinality of $B$ is 1000. We repeat the
estimation using different choices of $b$ and test its influence
on our estimate $\hat{r}(\boldsymbol{x};\hat{\boldsymbol{\delta}}_{b})$.
The results can be seen from Figure \ref{fig:Outlier-Setting}, where
the histograms of $G$ and $X_{q}$ are colored red and green respectively.
The true density ratio $\frac{p(x)}{q(x)}$ is plotted as a dotted
line. The histograms of $B_{b}$ with different choices of $b$ was
plotted using gradient colors from light blue to purple (we skipped
some choices of $b$ for better visualization). For each $b$, we
run the density ratio estimation, and plot learned $\hat{r}(\boldsymbol{x};\hat{\boldsymbol{\delta}}_{b})$
using the same gradient color. In the figure, we resale $\hat{r}(\boldsymbol{x};\hat{\boldsymbol{\delta}}_{b})$
and the true density ratio using a same constant, so they can be plotted
alongside with the histogram. Here, we test two methods: the log-Linear KLIEP and the robust estimator proposed in this paper.

It can be easily seen that as $b\rightarrow6$, KLIEP (Figure \ref{fig:Non-robust-Density-Ratio})
tends to significantly overestimate the density ratio and is sensitive
to the change of $b$. The proposed method (Figure \ref{fig:Robust-Density-Ratio}),
tends to underestimate the density ratio when $b$ is small. However,
as $b$ gradually shifts away from the center of $X_{p}$, leaving
the ``gap'' between inlier and outlier, the robust estimator converges
to the true density ratio function.

\paragraph{Truncated Setting}

In this setting, we generate samples $X_{p}\iid p(x)=N(0,1)$ without
any contamination. Usually, the $\nu$-th quantile of $z(\boldsymbol{x}_{p};\boldsymbol{\delta}^{*})$
cannot be analytically computed as we do not know the true density
ratio. However, it can be seen that for a strictly monotone increasing
$z(\boldsymbol{x}_{p},\boldsymbol{\delta}^{*})$, samples in the $\nu$-th
quantile of $z(\boldsymbol{x}_{p},\boldsymbol{\delta}^{*})$ must
be in the $\nu$-th quantile of $x_{p}$ since the relative
order among $x_{p}$ is preserved after a strictly monotone
transform. Thus, we obtain the truncation domain $\overline{X}(\boldsymbol{\delta}^{*})=\left\{ -\infty\le x\le\Phi^{-1}(\nu)\right\} $,
where $\Phi^{-1}$ is the inverse CDF of $N(0,1)$. We then generate
samples $X_{q}\sim TN(-0.5,1,-\infty,\Phi^{-1}(\nu))$, where $TN$
is a truncated Gaussian distribution and the last two parameters are
the truncation borders. Note we set the mean of $Q$ to be a negative
value so that the true density ratio $\bar{p}/\bar{q}$ is a monotone
increasing function.

The results for $\nu=0.5$ are plotted on Figure \ref{fig:Truncated-Setting}
where the true truncated ratio is plotted as a dotted line. It can
be seen that the learned $\hat{r}(\boldsymbol{x};\hat{\boldsymbol{\delta}})$
is fairly close to the true truncated density ratio.

\begin{figure}[t]
	\subfloat[Non-robust Density Ratio Estimation\label{fig:Non-robust-Density-Ratio}]{ \includegraphics[width=0.49\textwidth]{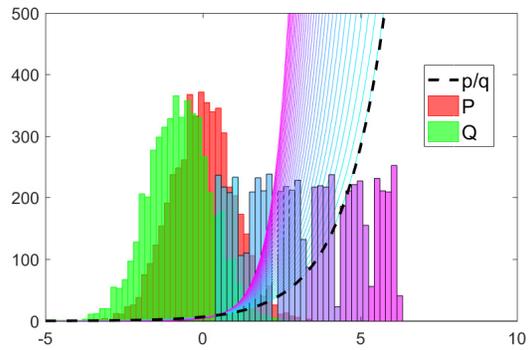}
	}
	\subfloat[Robust Density Ratio Estimation\label{fig:Robust-Density-Ratio}]{\includegraphics[width=0.49\textwidth]{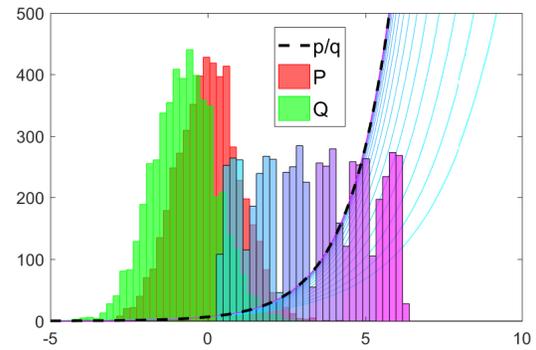}
	}
	\centering{}\caption{Outlier Setting\label{fig:Outlier-Setting}}
\end{figure}
\begin{figure}[t]
	\centering{}\includegraphics[width=0.7\textwidth]{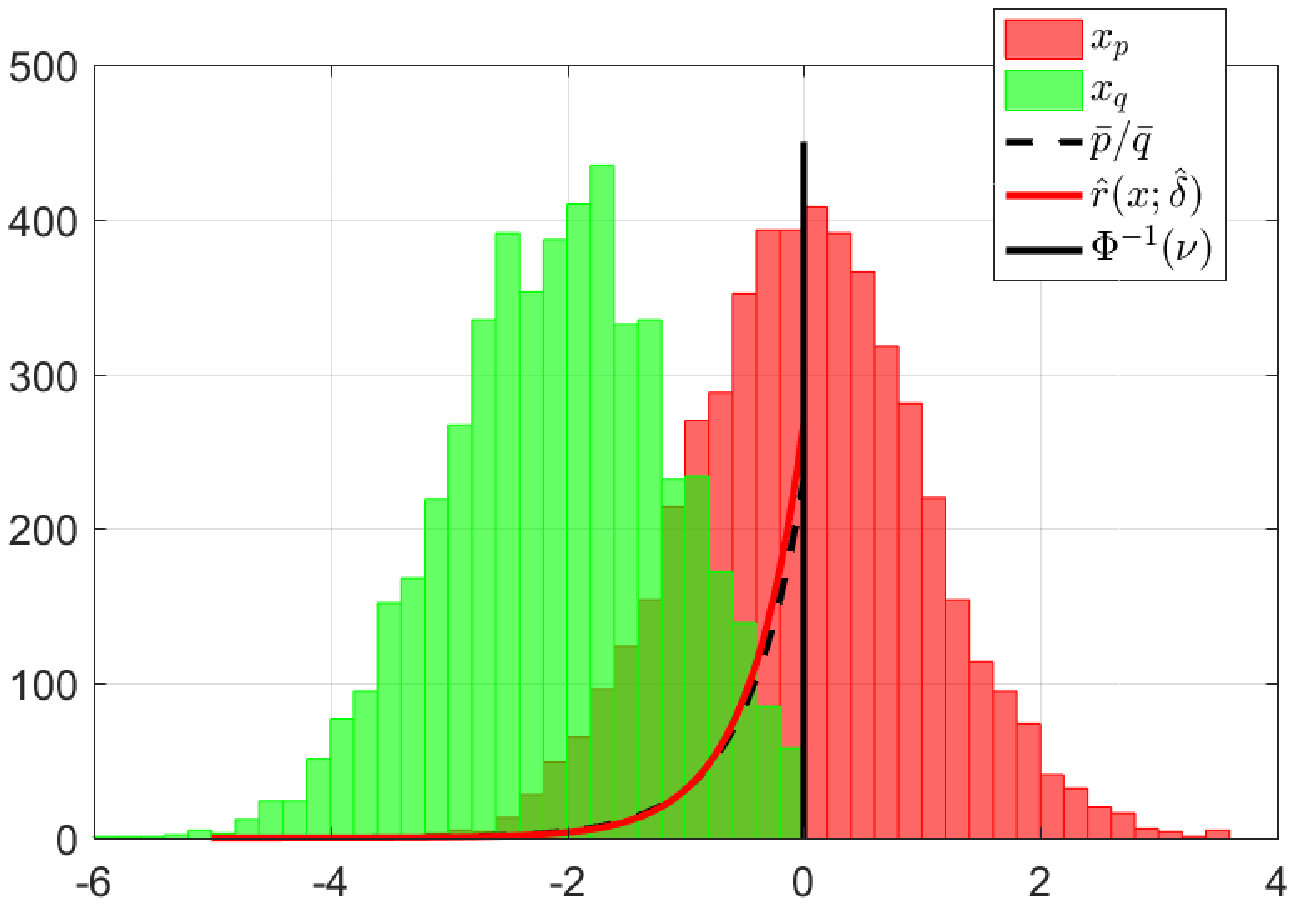}\caption{Truncated Set\label{fig:Truncated-Setting}ting}
\end{figure}

\end{document}

%% file: notation_song.tex
\newcommand{\boldtheta}{{\boldsymbol{\theta}}}
\newcommand{\bolddelta}{{\boldsymbol{\delta}}}

\newcommand{\boldf}{{\boldsymbol{f}}}
\newcommand{\boldu}{{\boldsymbol{u}}}

\newcommand{\boldone}{{\boldsymbol{1}}}

\newcommand{\boldDelta}{{\boldsymbol{\Delta}}}

\newcommand{\boldx}{{\boldsymbol{x}}}

\newcommand{\boldw}{{\boldsymbol{w}}}

\newcommand{\boldzero}{{\boldsymbol{0}}}

\newcommand{\boldX}{{\boldsymbol{X}}}

\newcommand{\iid}{\stackrel{\mathrm{i.i.d.}}{\sim}}

\newcommand{\vertiii}[1]{{\left\vert\kern-0.25ex\left\vert\kern-0.25ex\left\vert #1 
    \right\vert\kern-0.25ex\right\vert\kern-0.25ex\right\vert}}

%% file: main.bbl
\begin{thebibliography}{10}

\bibitem{azmandian2012local}
F.~Azmandian, J.~G. Dy, J.~A. Aslam, and D.~R. Kaeli.
\newblock Local kernel density ratio-based feature selection for outlier
  detection.
\newblock In {\em Proceedings of 8th Asian Conference on Machine Learning
  (ACML2012), JMLR Workshop and Conference Proceedings}, pages 49--64, 2012.

\bibitem{Boyd2014}
S.~Boyd.
\newblock Subgradient methods.
\newblock Technical report, Stanford University, 2014.
\newblock Notes for EE364b, Stanford University, Spring 2013–14.

\bibitem{Cleveland1979}
W.~S. Cleveland.
\newblock Robust locally weighted regression and smoothing scatterplots.
\newblock {\em Journal of the American Statistical Association},
  74(368):829--836, 1979.

\bibitem{Cristianini2000}
N.~Cristianini and J.~Shawe-Taylor.
\newblock {\em An Introduction to Support Vector Machines and Other
  Kernel-based Learning Methods}.
\newblock Cambridge University Press, 2000.

\bibitem{Efron1996}
B.~Efron and R.~Tibshirani.
\newblock Using specially designed exponential families for density estimation.
\newblock {\em The Annals of Statistics}, 24(6):2431--2461, 1996.

\bibitem{Fazayeli2016}
F.~Fazayeli and A.~Banerjee.
\newblock Generalized direct change estimation in ising model structure.
\newblock In {\em Proceedings of The 33rd International Conference on Machine
  Learning (ICML2016)}, page 2281–2290, 2016.

\bibitem{Fithian2015}
W.~Fithian and S.~Wager.
\newblock Semiparametric exponential families for heavy-tailed data.
\newblock {\em Biometrika}, 102(2):486--493, 2015.

\bibitem{Fokianos2004}
K.~Fokianos.
\newblock Merging information for semiparametric density estimation.
\newblock {\em Journal of the Royal Statistical Society: Series B (Statistical
  Methodology)}, 66(4):941--958, 2004.

\bibitem{Goodfellow2014}
I.~Goodfellow, J.~Pouget-Abadie, M.~Mirza, B.~Xu, D.~Warde-Farley, S.~Ozair,
  A.~Courville, and Y.~Bengio.
\newblock Generative adversarial nets.
\newblock In {\em Advances in neural information processing systems}, pages
  2672--2680, 2014.

\bibitem{Hadi1997}
A.~S. Hadi and A.~Luceno.
\newblock Maximum trimmed likelihood estimators: a unified approach, examples,
  and algorithms.
\newblock {\em Computational Statistics \& Data Analysis}, 25(3):251 -- 272,
  1997.

\bibitem{Huang2007}
J.~Huang, A.~Gretton, K.~M Borgwardt, B.~Sch{\"o}lkopf, and A.~J Smola.
\newblock Correcting sample selection bias by unlabeled data.
\newblock In {\em Advances in neural information processing systems}, pages
  601--608, 2007.

\bibitem{Huber1964}
P.~J. Huber.
\newblock Robust estimation of a location parameter.
\newblock {\em The Annals of Mathematical Statistics}, 35(1):73--101, 03 1964.

\bibitem{Kawahara2012}
Y.~Kawahara and M.~Sugiyama.
\newblock Sequential change-point detection based on direct density-ratio
  estimation.
\newblock {\em Statistical Analysis and Data Mining}, 5(2):114--127, 2012.

\bibitem{Liu2016a}
S.~Liu, T.~Suzuki, R.~Relator, J.~Sese, M.~Sugiyama, and K.~Fukumizu.
\newblock Support consistency of direct sparse-change learning in {Markov}
  networks.
\newblock {\em Annals of Statistics}, 45(3):959--990, 06 2017.

\bibitem{Loh2015}
P.-L. Loh and M.~J. Wainwright.
\newblock Regularized m-estimators with nonconvexity: Statistical and
  algorithmic theory for local optima.
\newblock {\em Journal of Machine Learning Research}, 16:559--616, 2015.

\bibitem{Nedic2009}
A.~Nedi{\'{c}} and A.~Ozdaglar.
\newblock Subgradient methods for saddle-point problems.
\newblock {\em Journal of Optimization Theory and Applications},
  142(1):205--228, 2009.

\bibitem{Neykov1990}
N.~Neykov and P.~N. Neytchev.
\newblock Robust alternative of the maximum likelihood estimators.
\newblock {\em COMPSTAT'90, Short Communications}, pages 99--100, 1990.

\bibitem{Nguyen2010}
X.~Nguyen, M.~J. Wainwright, and M.~I. Jordan.
\newblock Estimating divergence functionals and the likelihood ratio by convex
  risk minimization.
\newblock {\em IEEE Transactions on Information Theory}, 56(11):5847--5861,
  2010.

\bibitem{Nowozin2016}
S.~Nowozin, B.~Cseke, and R.~Tomioka.
\newblock f-gan: Training generative neural samplers using variational
  divergence minimization.
\newblock In {\em Advances in Neural Information Processing Systems}, pages
  271--279, 2016.

\bibitem{Pitman1936}
E.~J.~G. Pitman.
\newblock Sufficient statistics and intrinsic accuracy.
\newblock {\em Mathematical Proceedings of the Cambridge Philosophical
  Society}, 32(4):567–579, 1936.

\bibitem{Raskutti2010}
G.~Raskutti, M.~J. Wainwright, and B.~Yu.
\newblock Restricted eigenvalue properties for correlated gaussian designs.
\newblock {\em Journal of Machine Learning Research}, 11:2241--2259, 2010.

\bibitem{Rudelson2013}
M.~Rudelson and S.~Zhou.
\newblock Reconstruction from anisotropic random measurements.
\newblock {\em IEEE Transactions on Information Theory}, 59(6):3434--3447,
  2013.

\bibitem{Scholkopf2001}
B.~Scholkopf and A.~J. Smola.
\newblock {\em Learning with kernels: support vector machines, regularization,
  optimization, and beyond}.
\newblock MIT press, 2001.

\bibitem{Schoelkopf2000}
B.~Sch\"{o}lkopf, R.~C. Williamson, Smola~A. J., Shawe-Taylor J., and Platt
  J.C.
\newblock Support vector method for novelty detection.
\newblock In {\em Advances in Neural Information Processing Systems 12}, pages
  582--588. MIT Press, 2000.

\bibitem{Shimodaira2000}
A.~Shimodaira.
\newblock Improving predictive inference under covariate shift by weighting the
  log-likelihood function.
\newblock {\em Journal of Statistical Planning and Inference}, 90(2):227 --
  244, 2000.

\bibitem{Smola2009}
A.~Smola, L.~Song, and C.~H. Teo.
\newblock Relative novelty detection.
\newblock In {\em Proceedings of the Twelth International Conference on
  Artificial Intelligence and Statistics (AISTATS)}, volume~5, pages 536--543,
  2009.

\bibitem{Sugiyama2012}
M.~Sugiyama, T.~Suzuki, and T.~Kanamori.
\newblock {\em Density Ratio Estimation in Machine Learning}.
\newblock Cambridge University Press, 2012.

\bibitem{Sugiyama2008}
M.~Sugiyama, T.~Suzuki, S.~Nakajima, H.~Kashima, P.~von B\"unau, and
  M.~Kawanabe.
\newblock Direct importance estimation for covariate shift adaptation.
\newblock {\em Annals of the Institute of Statistical Mathematics},
  60(4):699--746, 2008.

\bibitem{Suykens2002}
J.~A.~K. Suykens, J.~De~Brabanter, L.~Lukas, and J.~Vandewalle.
\newblock Weighted least squares support vector machines: robustness and sparse
  approximation.
\newblock {\em Neurocomputing}, 48(1):85--105, 2002.

\bibitem{Tsuboi2009}
Y.~Tsuboi, H.~Kashima, S.~Hido, S.~Bickel, and M.~Sugiyama.
\newblock Direct density ratio estimation for large-scale covariate shift
  adaptation.
\newblock {\em Journal of Information Processing}, 17:138--155, 2009.

\bibitem{Vandev1998}
D.~L. Vandev and N.~M. Neykov.
\newblock About regression estimators with high breakdown point.
\newblock {\em Statistics: A Journal of Theoretical and Applied Statistics},
  32(2):111--129, 1998.

\bibitem{Wornowizki2016}
M.~Wornowizki and R.~Fried.
\newblock Two-sample homogeneity tests based on divergence measures.
\newblock {\em Computational Statistics}, 31(1):291--313, 2016.

\bibitem{Yamada2013}
M.~Yamada, T.~Suzuki, T.~Kanamori, H.~Hachiya, and M.~Sugiyama.
\newblock Relative density-ratio estimation for robust distribution comparison.
\newblock {\em Neural Computation}, 25(5):1324--1370, 2013.

\bibitem{yang2016high}
E.~Yang, A.~Lozano, and A.~Aravkin.
\newblock High-dimensional trimmed estimators: A general framework for robust
  structured estimation.
\newblock {\em arXiv preprint arXiv:1605.08299}, 2016.

\bibitem{Yang2015}
E.~Yang and A.~C. Lozano.
\newblock Robust gaussian graphical modeling with the trimmed graphical lasso.
\newblock In {\em Advances in Neural Information Processing Systems}, pages
  2602--2610, 2015.

\end{thebibliography}
